%% file: main.tex
\documentclass[a4paper,11pt]{article}

\usepackage{publication}
\usepackage{amssymb}
\usepackage{amsthm}
\usepackage{booktabs}
\usepackage[ruled,vlined]{algorithm2e}
\usepackage{enumitem}
\usepackage[backend=biber]{biblatex}
\newtheorem{proposition}{Proposition}
\newtheorem{definition}{Definition}
\newtheorem{remark}{Remark}

\title{Mixed Integer Goal Programming for Personalized Meal Optimization with User-Defined Serving Granularity}

\author{Francisco Aguilera Moreno}
\contact{aguimorefran@gmail.com}

\date{March 2026}

\begin{document}

\maketitle

\begin{abstract}
    \input{sections/abstract}
\end{abstract}

\input{sections/introduction}
\input{sections/related_work}
\input{sections/formulation}
\input{sections/integrality_analysis}
\input{sections/implementation}
\input{sections/experiments}
\input{sections/discussion}

\printbibliography

\end{document}

%% file: sections/abstract.tex
Determining what to eat to satisfy nutritional requirements is one of the
oldest optimization problems in operations research, yet existing formulations
have two persistent limitations: continuous variables produce impractical
fractional servings (1.7~eggs, 0.37~bananas), and hard nutrient constraints
cause infeasibility when targets conflict. A systematic review of 56 diet
optimization papers found that none combine integer programming with goal
programming to address both issues.

We propose Mixed Integer Goal Programming (MIGP) for personalized meal
optimization. The formulation uses integer variables for practical serving
counts and goal programming deviations for soft nutrient targets, with
inverse-target normalization to balance multi-nutrient optimization.
Per-food serving granularity allows natural units (one egg, one tablespoon
of oil) without post-hoc rounding.

We characterize the integrality gap in the goal programming context and
identify a \emph{deviation absorption} property: GP deviation variables
buffer the cost of requiring integer servings, making the gap structurally
smaller than in hard-constraint MIP. For meals with 15+ foods, the integer
solution matches the continuous optimum in every benchmark instance.

A computational evaluation across 810 instances (30 USDA foods,
9~configurations, 3~methods) shows MIGP finds strictly better solutions
than GP with post-hoc rounding in 66\% of cases (never worse) while
maintaining 100\% feasibility; hard-constraint IP achieves only 48\%.
Solve times stay under 100\,ms for typical meal sizes using the open-source
HiGHS solver. The implementation is available as an open-source Python
module integrated into an interactive meal planning application.

%% file: sections/introduction.tex
\section{Introduction}
\label{sec:introduction}

\subsection{The Diet Problem}

Determining what to eat to satisfy nutritional requirements is one of
the oldest optimization problems in operations research.
Stigler~\cite{stigler1945cost} formulated the minimum-cost diet problem
in 1945; Dantzig~\cite{dantzig1963linear} solved it with the simplex
method shortly after. Nearly 80 years later, the problem remains
practically relevant: health tracking apps, meal prep services, and
clinical nutrition programs all need to compose meals that meet
individualized nutritional targets.

The mathematical structure generalizes beyond nutrition.
\emph{Resource-composition problems} arise whenever discrete quantities
of inputs must be combined to approximate a target composition vector:
blending chemicals in manufacturing, constructing a portfolio from
integer share lots, or loading vehicles to balance weight across
dimensions. In each case, inputs come in discrete units, targets span
multiple dimensions, and exact achievement is typically impossible.

\subsection{Two Persistent Limitations}

Despite decades of refinement, diet optimization faces two limitations
rooted in the standard LP formulation.

\paragraph{Fractional servings.}
Continuous variables produce solutions like 1.7~eggs or
0.37~bananas, mathematically precise but useless for meal preparation.
The literature has largely treated this as an implementation detail,
rounding to integers after the fact~\cite{gazan2018mathematical}. But
rounding is not neutral: it changes the nutritional profile without
regard for optimality, and the rounded solution may violate the very
constraints that motivated the optimization.

Integer programming solves this directly. However, existing IP
formulations for diet use binary selection variables (include/exclude a
food)~\cite{benvenuti2024designing} rather than integer quantity
variables (how many servings). The question of \emph{how much} of each
food (expressed in whole servings) has not been addressed with integer
optimization.

\paragraph{Hard constraints cause infeasibility.}
When nutrient requirements conflict or the food set is limited,
hard-constraint formulations declare the problem infeasible and return
nothing. Briend et al.~\cite{briend2003linear} documented this in the
WHO context, where iron requirements could only be met at 63\% of
recommendations. In interactive applications this is problematic: the system provides no
guidance precisely when guidance is most needed.

Goal programming~\cite{charnes1961management} eliminates infeasibility
by replacing hard constraints with soft targets. Deviation variables
measure how far the solution falls from each target, and the objective
minimizes total deviation. Gerdessen and de
Vries~\cite{gerdessen2015diet} applied GP to diet optimization with
good results, but their formulation uses continuous variables, inheriting
the fractional serving problem.

\subsection{The Gap}

A systematic review by Donkor et al.~\cite{donkor2023systematic}
examined 56 diet optimization papers and found that none combine integer
programming with goal programming (detailed analysis in
Section~\ref{sec:related}). Concurrently, AI/ML approaches have emerged
for meal planning~\cite{nutrigen2025, cfrl2024, vanwonderen2026diet},
offering preference learning and generative capabilities but without
formal optimality guarantees or transparent tradeoff reporting. These
approaches are complementary to exact optimization.

\subsection{Contributions}

We propose \emph{Mixed Integer Goal Programming} (MIGP) for
personalized meal optimization, filling the gap identified by Donkor
et al. Our contributions:

\begin{enumerate}
\item \textbf{Novel formulation.} The first unified MIGP model for
  diet optimization, combining integer serving variables with GP
  deviation minimization. The model supports user-defined per-food
  serving granularity and inverse-target penalty normalization for
  balanced multi-nutrient optimization
  (Section~\ref{sec:formulation}).

\item \textbf{Integrality analysis.} We characterize the integrality
  gap in the GP context, identifying a \emph{deviation absorption}
  property where deviation variables buffer the cost of integrality.
  With 15+ foods, the integer solution matches the continuous optimum.
  This property generalizes to any GP formulation with integer
  variables, not just diet optimization
  (Section~\ref{sec:integrality}).

\item \textbf{Open-source implementation.} A Python solver using HiGHS
  achieves sub-100\,ms solve times for typical meal sizes, integrated
  into an interactive web application
  (Section~\ref{sec:implementation}).

\item \textbf{Comparative evaluation.} A benchmark across 810 instances
  shows MIGP finds strictly better solutions than GP+rounding in 66\%
  of cases (never worse) while maintaining 100\% feasibility; hard-constraint
  IP achieves only 48\% (Section~\ref{sec:experiments}).
\end{enumerate}

\subsection{Paper Organization}

Section~\ref{sec:related} reviews related work.
Section~\ref{sec:formulation} presents the MIGP formulation.
Section~\ref{sec:integrality} analyzes the integrality gap.
Section~\ref{sec:implementation} describes the implementation.
Section~\ref{sec:experiments} reports computational experiments.
Section~\ref{sec:discussion} discusses implications, limitations, and
future directions.

%% file: sections/related_work.tex
\section{Related Work}
\label{sec:related}

\subsection{Classical Diet Optimization}

The diet problem is one of the earliest applications of mathematical
programming. Stigler~\cite{stigler1945cost} formulated the problem of
finding a minimum-cost diet satisfying nutritional requirements for
77~foods and 9~nutrients, solving it by manual enumeration.
Dantzig~\cite{dantzig1963linear} later applied the simplex method to this
problem, establishing diet optimization as a canonical linear programming
application.

Subsequent work refined the LP approach for specific populations.
Briend et al.~\cite{briend2003linear} applied LP to complementary feeding
for infants in a WHO context, demonstrating a fundamental limitation:
when iron requirements could only be met at 63\% of recommendations, the
model declared infeasibility rather than providing a best-effort solution.
Maillot et al.~\cite{maillot2010individual} developed individual diet
modeling that translates nutrient recommendations into personalized food
choices, and van Dooren et al.~\cite{vandooren2014exploring,
vandooren2015combining} extended the framework to incorporate
sustainability criteria (cost and climate impact) alongside nutrition.

Throughout this evolution, the mathematical structure remained the same:
continuous LP with hard nutrient constraints. Two persistent limitations
emerge. First, continuous variables produce impractical fractional
servings (0.37~bananas, 1.7~eggs) that require manual rounding. Second,
hard constraints cause infeasibility when nutrient requirements conflict
or when the available food set cannot meet all targets simultaneously.

\subsection{Goal Programming for Diet}

Goal programming, introduced by Charnes and
Cooper~\cite{charnes1961management}, replaces hard constraints with soft
targets and explicit deviation variables, eliminating the infeasibility
problem. Tamiz et al.~\cite{tamiz1998goal} provide an overview of GP
methodology, and Romero~\cite{romero2004general} formalizes the general
achievement function structure.

Gerdessen and de Vries~\cite{gerdessen2015diet} applied GP to diet
optimization with 144 Dutch foods and 19 nutrients, comparing three
achievement functions: MinSum (weighted $L_1$), MinMax ($L_\infty$), and
Extended GP (combining both). Their key finding is that the choice of
achievement function significantly impacts diet composition: MinSum
favors well-rounded meals while MinMax prevents extreme single-nutrient
deviations. However, their formulation uses \emph{continuous} decision
variables, producing fractional food quantities that are impractical for
meal preparation.

Our MIGP formulation builds directly on Gerdessen and de Vries' GP
framework, adopting the MinSum achievement function while adding integer
decision variables for practical serving counts. The integrality analysis
in Section~\ref{sec:integrality} quantifies the cost of this extension.

\subsection{Integer Programming for Diet}

Integer programming in diet optimization has focused primarily on binary
selection variables (include/exclude a food) rather than integer quantity
variables. Benvenuti et al.~\cite{benvenuti2024designing} formulate a
triobjective 0--1 integer program for school cafeteria menus, using binary
variables with hard nutritional constraints. Their TrIntOpt algorithm
enumerates Pareto-optimal menus, but the hard-constraint formulation
suffers from infeasibility when objectives conflict.

Gazan et al.~\cite{gazan2018mathematical} review 67 studies using LP for
diet optimization and note that integer LP is ``more pragmatic'' for
real-world application but ``computationally intensive.'' They document the
common practice of Optifood and similar tools: solve a continuous LP, then
round the solution post-hoc. This GP+rounding approach is the baseline we
evaluate against in Section~\ref{sec:experiments}.

The gap between the GP literature (continuous, always feasible, soft
targets) and the IP literature (integer, frequently infeasible, hard
constraints) motivates our formulation: MIGP combines the strengths of
both paradigms.

\subsection{AI and Machine Learning Approaches}

Recent work applies AI/ML techniques to meal planning.
Khamesian et al.~\cite{nutrigen2025} use large language models (NutriGen),
Amiri et al.~\cite{cfrl2024} combine reinforcement learning with
collaborative filtering, and van Wonderen et
al.~\cite{vanwonderen2026diet} pair recipe completion algorithms with
diet optimization. These approaches excel at learning user preferences
and generating novel food combinations, but share two limitations
relative to mathematical programming: (1)~no formal guarantee of
optimality---the returned meal may not minimize deviation from nutrient
targets---and (2)~no guarantee of constraint satisfaction---hard bounds
on nutrients or servings may be violated silently. MIGP and AI/ML are
therefore complementary: MIGP provides optimality, feasibility, and
transparency of tradeoffs, while ML provides capabilities that MIGP does
not.

\subsection{Structural Analogies}

The MIGP formulation belongs to a broader class of resource-composition
problems where discrete quantities of input resources must approximate a
target composition vector.

The \emph{multi-dimensional knapsack problem}~\cite{kellerer2004knapsack}
selects integer quantities subject to capacity constraints in multiple
dimensions. Diet MIGP has the same structure: select integer servings
subject to soft nutrient targets across multiple macronutrients.
Concretely, choosing how many bags of each fertilizer to buy to
approximate a target NPK (nitrogen-phosphorus-potassium) ratio is
structurally identical to choosing serving counts to approximate
macro targets. The knapsack is NP-hard in general, though the small
problem sizes in meal planning (8--25 foods) keep computation tractable.

In \emph{manufacturing blending} (coal, tea, gasoline, animal feed),
quantities of raw ingredients are combined to achieve a target composition.
For example, blending integer crates of different tea varieties to achieve
a target caffeine-to-tannin ratio parallels blending integer food servings
to achieve a target protein-to-carb ratio. This is a classical GP
application~\cite{romero2004general}; our formulation adds integrality
when raw ingredients come in discrete units.

\emph{Portfolio selection with cardinality
constraints}~\cite{bonami2009exact} requires integer numbers of shares
subject to multi-objective targets (return, risk). The mathematical
structure (integer quantities of discrete resources meeting soft
composition targets) is identical to diet MIGP. Deviation
absorption applies here as well: soft return and risk targets with GP
deviation variables would buffer the cost of integer share lots, just as
deviation variables buffer the cost of integer servings in our formulation.

A key insight from our integrality analysis is that the deviation
absorption property (Section~\ref{sec:absorption}) is not specific to
diet optimization. Any GP formulation with integer decision variables
will exhibit this behavior, suggesting that the integrality gap in GP
contexts is structurally smaller than in hard-constraint MIP across all
these domains. We are not aware of prior work characterizing this
phenomenon.

\subsection{The Gap}

Donkor et al.~\cite{donkor2023systematic} provide the most comprehensive
recent review, examining 56 diet optimization papers and classifying them
by mathematical approach. Their taxonomy reveals that 54 of 56 papers use
standard single-objective LP; Gerdessen and de
Vries~\cite{gerdessen2015diet} use GP but with continuous variables;
Benvenuti et al.~\cite{benvenuti2024designing} use IP but with binary
variables and hard constraints. No prior work combines integer
programming with goal programming. Donkor et al.\ explicitly call for
``new mathematical approaches'' to address the limitations of existing
methods.

Our independent search of the 2024--2026 literature (50+ additional
papers) confirms this gap remains open. Bashiri et
al.~\cite{bashiri2025sustainable} combine multi-objective optimization
with multi-criteria decision making for sustainable diets, but do not use
goal programming or integer serving variables. The AI/ML approaches
reviewed above operate in a different paradigm. The MIGP
formulation presented in this work is, to our knowledge, the first to
fill the gap identified by Donkor et al.

Table~\ref{tab:methods} summarizes the positioning of MIGP relative to
existing approaches.

\begin{table}[H]
\centering
\small
\begin{tabular}{@{}l c c c c c@{}}
\toprule
\textbf{Approach} & \textbf{Variables} & \textbf{Constraints} & \textbf{Feasible} & \textbf{Servings} & \textbf{Optimal} \\
\midrule
Classical LP    & Continuous & Hard  & No  & Fractional & Yes \\
GP (Gerdessen)  & Continuous & Soft  & Yes & Fractional & Yes \\
IP (Benvenuti)  & Binary     & Hard  & No  & Binary     & Yes \\
AI/ML           & ---        & Soft  & Yes & Varies     & No  \\
\textbf{MIGP (ours)} & \textbf{Integer} & \textbf{Soft} & \textbf{Yes} & \textbf{Integer} & \textbf{Yes} \\
\bottomrule
\end{tabular}
\caption{Comparison of diet optimization approaches. MIGP uniquely
combines integer serving variables with soft goal constraints, achieving
both feasibility and optimality with practical serving outputs.}
\label{tab:methods}
\end{table}

%% file: sections/formulation.tex
\section{Mathematical Formulation}
\label{sec:formulation}

\subsection{Problem Definition}

Given a set of foods with known nutritional composition and user-defined
serving sizes, the meal optimization problem seeks integer serving counts
that minimize deviation from macronutrient targets. Unlike classical diet
formulations that impose hard constraints on nutrient
intake~\cite{stigler1945cost, dantzig1963linear}, we adopt a goal
programming approach where each target becomes a soft goal with explicit
deviation variables~\cite{charnes1961management}. Combined with integer
decision variables for practical serving counts, this yields a Mixed
Integer Goal Program (MIGP).

\subsection{Notation}

Table~\ref{tab:notation} summarizes the notation used throughout.

\begin{table}[H]
\centering
\begin{tabular}{@{}l l@{}}
\toprule
\textbf{Symbol} & \textbf{Description} \\
\midrule
\multicolumn{2}{@{}l}{\textit{Sets}} \\
$\mathcal{F} = \{1, \ldots, n\}$ & Selected foods \\
$\mathcal{M} = \{\text{cal, prot, carbs, fat}\}$ & Macronutrients \\
\midrule
\multicolumn{2}{@{}l}{\textit{Parameters}} \\
$a_{i,m}$ & Nutrient $m$ per 100\,g of food $i$ (kcal or grams) \\
$s_i$ & Serving size in grams for food $i$ \\
$\ell_i, u_i$ & Minimum and maximum integer servings for food $i$ \\
$T_m$ & Target value for macronutrient $m$ \\
$w_m$ & Penalty weight for macronutrient $m$ \\
\midrule
\multicolumn{2}{@{}l}{\textit{Derived}} \\
$c_{i,m} = a_{i,m} \cdot s_i \,/\, 100$ & Nutrient $m$ per serving of food $i$ \\
\midrule
\multicolumn{2}{@{}l}{\textit{Decision variables}} \\
$x_i \in \mathbb{Z}_{\geq 0}$ & Integer serving count for food $i$ \\
$d_m^+, d_m^- \geq 0$ & Over- and under-deviation for macronutrient $m$ \\
\bottomrule
\end{tabular}
\caption{Notation used in the MIGP formulation.}
\label{tab:notation}
\end{table}

The user selects foods from a nutritional database where values are
stored per 100\,g. Each food has a configurable serving size $s_i$
(e.g., 60\,g for an egg, 150\,g for a chicken breast), so granularity
reflects how foods are actually consumed. Serving bounds let the user
enforce preferences like ``at least 1 serving of rice'' or ``at most
3 eggs.''

\subsection{Target Derivation}

The user specifies a calorie target and macronutrient percentages
(protein, carbs, fat) that sum to~100\%.
Gram targets follow from standard energy densities (4\,kcal/g for
protein and carbohydrates, 9\,kcal/g for fat):
\begin{align}
T_{\text{prot}} &= \frac{T_{\text{cal}} \cdot p_{\text{prot}}}{100 \times 4}
\label{eq:target_prot} \\
T_{\text{carbs}} &= \frac{T_{\text{cal}} \cdot p_{\text{carbs}}}{100 \times 4}
\label{eq:target_carbs} \\
T_{\text{fat}} &= \frac{T_{\text{cal}} \cdot p_{\text{fat}}}{100 \times 9}
\label{eq:target_fat}
\end{align}

A 600\,kcal target with a 30/45/25 split gives 45\,g protein,
67.5\,g carbs, and 16.7\,g fat. The calorie target enters directly
as 600\,kcal.

\subsection{Goal Constraints}

Each macronutrient has a goal constraint linking food servings to the
target through deviation variables:
\begin{equation}
\sum_{i \in \mathcal{F}} c_{i,m} \, x_i \;+\; d_m^- \;-\; d_m^+ = T_m,
\quad \forall\, m \in \mathcal{M}
\label{eq:goal}
\end{equation}
The variable $d_m^-$ captures any shortfall below the target; $d_m^+$
captures any excess above it. At optimality, at most one is positive for
each macronutrient. Having both positive would mean we could reduce them
without worsening anything.

This follows the standard goal programming paradigm of Charnes and
Cooper~\cite{charnes1961management}: hard constraints become soft targets
with explicit deviation measurement. Feasibility implications are
established in Proposition~\ref{prop:feasibility} below.

\subsection{Objective Function}

The objective minimizes the weighted sum of all deviations, the MinSum
achievement function~\cite{romero2004general}:
\begin{equation}
\min \sum_{m \in \mathcal{M}} w_m \left( d_m^+ + d_m^- \right)
\label{eq:objective}
\end{equation}

Both over- and under-achievement are penalized equally: eating 20\,g
above or below a protein target incurs the same cost. This reflects
the meal planning context where exceeding a calorie target is as
undesirable as falling short. Alternative achievement functions (MinMax,
Extended GP) are compared by Gerdessen and de
Vries~\cite{gerdessen2015diet}; we adopt MinSum for simplicity.

\subsection{Penalty Weight Normalization}
\label{sec:weights}

Macronutrient targets span different scales: calorie targets are
typically 500--2000\,kcal, while fat targets may be 15--50\,g. Without
normalization, calorie deviations dominate the objective simply because
of their larger absolute magnitude.

We use inverse-target normalization:
\begin{equation}
w_m = \frac{1}{\max(T_m,\, 1)}
\label{eq:weight}
\end{equation}

This converts each deviation into a fraction of its target before
summing. A 10\,kcal miss on a 2000\,kcal target (0.5\%) gets the
same penalty as a 0.5\,g miss on a 100\,g protein target (also~0.5\%).
The $\max(\cdot, 1)$ guard handles zero targets, such as zero-carb diets.

Equal weights ($w_m = 1$) are the obvious alternative, but they bias
toward calorie accuracy at the expense of macronutrient balance.
Gerdessen and de Vries~\cite{gerdessen2015diet} compare several
achievement functions but do not address the normalization problem
introduced by integer variables. Sensitivity to weight choices is
evaluated in Section~\ref{sec:experiments}.

Concretely: a 45\,g protein target yields a weight about 13$\times$
larger than a 600\,kcal calorie target, so a 1-gram protein error is
penalized 13 times more than a 1-kcal calorie error, reflecting the
smaller absolute scale of protein.

An alternative to weight tuning is \emph{lexicographic} goal
programming, which solves a sequence of single-objective problems: first
minimize deviation on the highest-priority macro (e.g., protein for a
strength athlete), then minimize the next macro subject to keeping the
first at its optimum, and so on. This eliminates weight choice entirely
but imposes a rigid hierarchy: no amount of improvement in a
lower-priority macro can compensate for any worsening of a higher-priority
one. The weighted MinSum approach adopted here provides a more flexible
tradeoff, and the sensitivity analysis in Section~\ref{sec:sensitivity}
shows that weight adjustments give effective user control without the
rigidity of a strict lexicographic ordering.

\subsection{Complete Formulation}

Combining the elements above:
\begin{align}
\min \quad & \sum_{m \in \mathcal{M}} w_m \left( d_m^+ + d_m^- \right)
\label{eq:migp_obj} \\[4pt]
\text{s.t.} \quad
& \sum_{i \in \mathcal{F}} c_{i,m} \, x_i + d_m^- - d_m^+ = T_m,
  & \forall\, m \in \mathcal{M}
\label{eq:migp_goal} \\
& \ell_i \leq x_i \leq u_i,
  & \forall\, i \in \mathcal{F}
\label{eq:migp_bounds} \\
& x_i \in \mathbb{Z}_{\geq 0},
  & \forall\, i \in \mathcal{F}
\label{eq:migp_int} \\
& d_m^+,\, d_m^- \geq 0,
  & \forall\, m \in \mathcal{M}
\label{eq:migp_dev}
\end{align}
where $c_{i,m} = a_{i,m} \cdot s_i / 100$ and
$w_m = 1 \,/\, \max(T_m, 1)$.

For a typical meal with $n$ foods and 4~macronutrients, the model has
$n$~integer variables (serving counts) plus 8~continuous variables
(deviations), and just 4~equality constraints. A 15-food meal produces
a 23-variable, 4-constraint problem, small for modern MIP solvers.
Ours solves these in under 50\,ms
(Section~\ref{sec:implementation}).

\subsection{Properties}

We establish two foundational results: a feasibility guarantee that
sets MIGP apart from hard-constraint formulations, and a formal
definition of the integrality gap analyzed in
Section~\ref{sec:integrality}.

\begin{proposition}[Feasibility Guarantee]
\label{prop:feasibility}
The MIGP~\eqref{eq:migp_obj}--\eqref{eq:migp_dev} is feasible for any
target vector~$T$ and any food set~$\mathcal{F}$ with at least one food
satisfying $u_i \geq \ell_i$.
\end{proposition}

\begin{proof}
For any integer serving counts within the
bounds~\eqref{eq:migp_bounds}--\eqref{eq:migp_int}, let $A_m$ be the
achieved value of macro~$m$. The deviations can always absorb the gap:
set $d_m^+ = \max(0,\, A_m - T_m)$ and $d_m^- = \max(0,\, T_m - A_m)$.
At least one valid assignment exists (e.g., all foods at their minimum
serving), so the model always admits a feasible solution.
\end{proof}

MIGP never fails to return a solution. Hard-constraint approaches may declare
a problem infeasible and return nothing. Any mismatch between the food
selection and the targets is absorbed into the deviation variables. In our
benchmark, hard-constraint IP returns nothing for 51.9\% of instances;
MIGP succeeds on 100\%
(Table~\ref{tab:comparison})~\cite{briend2003linear}.

\begin{definition}[Integrality Gap]
\label{def:gap}
Let $Z_{\text{MIP}}$ denote the optimal objective of the
MIGP~\eqref{eq:migp_obj}--\eqref{eq:migp_dev}, and $Z_{\text{LP}}$ the
optimal value of its continuous relaxation (replacing
constraint~\eqref{eq:migp_int} with $x_i \in \mathbb{R}_+$). The
\emph{integrality gap} is:
\begin{equation}
\gamma = \frac{Z_{\text{MIP}} - Z_{\text{LP}}}{Z_{\text{LP}}}
\label{eq:gap}
\end{equation}
when $Z_{\text{LP}} > 0$. If $Z_{\text{LP}} = 0$ (the continuous
relaxation hits all targets exactly), we define $\gamma = 0$ when
$Z_{\text{MIP}} = 0$ and report the absolute difference otherwise.
\end{definition}

The integrality gap measures the cost of requiring whole servings
instead of fractions. In meal planning terms, it tells you how much
worse your macros get because you need 2~eggs instead of 1.7. The
MIGP solver considers the downstream effect on \emph{all}
macronutrients when choosing integer serving counts, not just the
individual rounding of each food.

Section~\ref{sec:integrality} characterizes $\gamma$ across problem
configurations and compares MIGP against the common practice of
solving a continuous GP and rounding
afterward~\cite{gazan2018mathematical}.

%% file: sections/integrality_analysis.tex
\section{Integrality Analysis}
\label{sec:integrality}

The integrality gap $\gamma$ (Definition~\ref{def:gap}) measures the
price of requiring whole servings. This section characterizes when the
gap is large or small, identifies a structural property of goal
programming that makes it behave differently from standard MIP, and
compares MIGP against the common practice of solving a continuous GP
followed by post-hoc rounding.

All empirical results here come from the benchmark described in
Section~\ref{sec:experiments}: 270 MIGP instances across 9
configurations (3 sizes, 3 difficulty levels) with 30 random seeds
each, using a 30-food bank with USDA nutritional values.

\subsection{Continuous Relaxation}

The continuous relaxation drops the integrality
constraint~\eqref{eq:migp_int}, allowing fractional servings and
yielding a standard LP. This relaxed model is equivalent to the
continuous GP of Gerdessen and de Vries~\cite{gerdessen2015diet}. Its
optimal value provides a lower bound on the integer objective, since
allowing fractions can only help (or at worst tie).

Two regimes emerge:

\paragraph{Perfect continuous solution (zero LP deviation).}
The LP finds fractional servings that hit all macro targets exactly.
This happens in 61\% of instances (165/270), mostly in loose and tight
configurations (91\% and 92\%) where serving bounds are not restrictive.
In all 165 cases, the integer objective is positive: requiring whole
servings always introduces some deviation when the continuous optimum
involves fractions. The median integer objective is 0.024, small but
nonzero.

\paragraph{Imperfect continuous solution (positive LP deviation).}
Even with fractional servings, the targets cannot be met exactly. This
occurs in all ambitious configurations (every food forced to have at
least 1~serving) and a minority of small instances.

Among the 105 instances where the LP cannot achieve zero deviation,
the integer solver matches the continuous optimum exactly in
83\% of cases (87/105). Integrality imposes no
additional cost. This is a consequence of \emph{deviation absorption}
(Section~\ref{sec:absorption}): the deviation variables provide enough
flexibility to redistribute rounding costs without increasing the total
weighted deviation.
\label{rem:tight_bound}

\subsection{Gap Characterization}
\label{sec:gap_char}

\subsubsection{Problem Size Is the Dominant Factor}

Figure~\ref{fig:gap_vs_foods} shows the integrality cost across problem
sizes. Panel~(a) reports the absolute cost (integer minus continuous
objective); panel~(b) reports the percentage gap for instances with
positive LP deviation. The pattern is stark: positive gaps occur
\emph{exclusively} in small instances (8~foods). With 15 or more
foods, the gap is zero across all 180 instances regardless of
configuration.

\begin{figure}[H]
\centering
\includegraphics[width=0.85\textwidth]{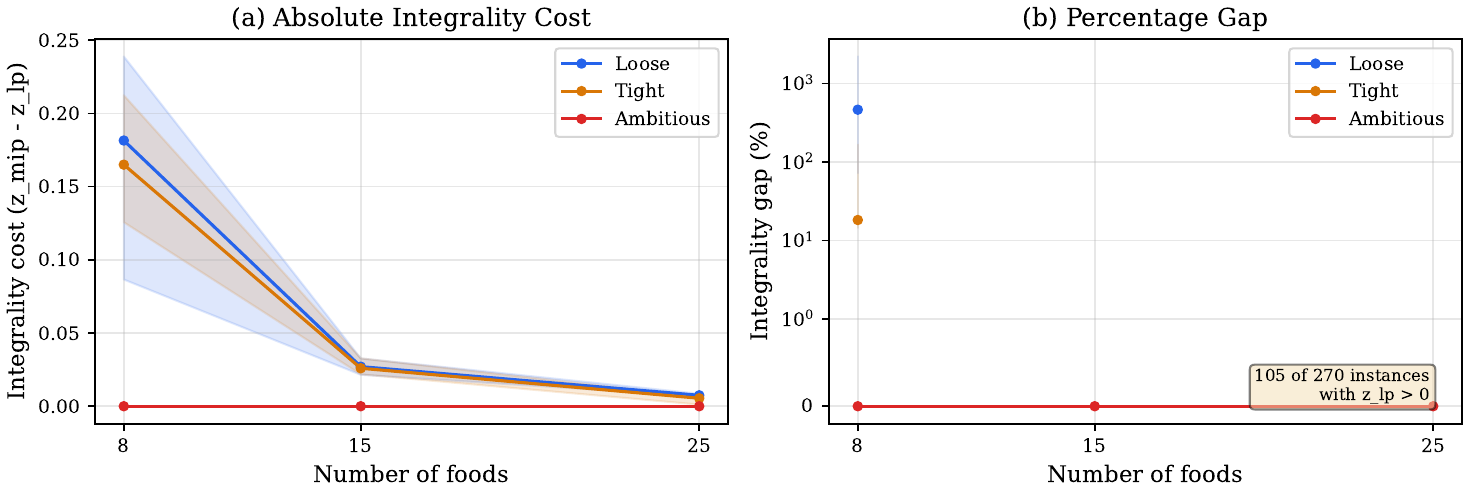}
\caption{Integrality cost vs.\ number of foods. \textbf{(a)}~Absolute
integrality cost (integer minus continuous objective) across all instances.
\textbf{(b)}~Percentage gap, computed only for instances where the LP
has positive deviation (annotation shows subsample size). Shaded
regions show interquartile ranges; dots indicate medians. Positive gaps
occur only with 8~foods; 15 and 25 food instances always achieve zero
integrality cost.}
\label{fig:gap_vs_foods}
\end{figure}

The intuition is simple: more foods mean more degrees of freedom. With
25~foods and only 4~macro constraints, the system is heavily
underdetermined: the solver can find integer combinations that match the
continuous optimum. With only 8~foods, the system is tighter and
rounding one food's servings may cascade into deviations across
multiple macros.

\subsubsection{Serving Bound Tightness Has Limited Effect}

Among configurations where positive gaps occur (small instances), the
serving range matters little. Loose bounds (0--10 servings) and tight
bounds (0--4) produce positive gaps in 27\% and 23\% of small instances
respectively, a negligible difference. The ambitious configuration
(1--3 servings required) rarely produces positive gaps despite its
narrow range, because targets are typically unachievable even
continuously and the LP bound is already tight.

\subsubsection{Extreme Percentage Gaps Are Misleading}

A few instances show very large percentage gaps (over~1000\%). These
occur when the LP deviation is near zero (e.g., 0.005), inflating the
denominator of Equation~\eqref{eq:gap}. The absolute rounding cost is
modest (below~0.25); the percentage is just an artifact of division by
a tiny number. For practical assessment, absolute deviation from
targets (Section~\ref{sec:experiments}) is more useful.

\subsection{Deviation Absorption in Goal Programming}
\label{sec:absorption}

The integrality gap in MIGP behaves differently from the gap in
standard MIP, and the reason is structural. In a standard MIP (e.g.,
the hard-constraint IP in our comparison), the objective depends
directly on the integer variables. Rounding worsens the objective with
no mechanism to compensate.

In MIGP, the deviation variables act as a \emph{buffer layer} between
the integer variables and the objective. When integrality forces a
serving count away from its continuous optimum, the deviations adjust
to absorb the change. The cost is not a direct loss but a
redistribution of deviations across macronutrients. This explains two
observations:

\begin{enumerate}
\item \textbf{Tight LP bounds.} In 83\% of instances with positive LP
deviation, the integer solver matches the continuous optimum exactly.
The deviations redistribute rounding costs across macros without
increasing the total. The MinSum objective allows trading a small
increase in one macro's deviation for a decrease in another.

\item \textbf{Rapid convergence with size.} More foods mean more
integer combinations and more ``adjustment room.'' The solver exploits
the many-to-many mapping between foods and macros: if rounding chicken
up by one serving adds excess protein, it can round lentils down to
compensate, with deviations absorbing the residual.
\end{enumerate}

We formalize this mechanism as follows.

\begin{proposition}[Deviation Absorption]
\label{prop:absorption}
Let $x^*_{\text{LP}}$ be an optimal continuous solution to the
relaxation of~\eqref{eq:migp_obj}--\eqref{eq:migp_dev}, with optimal
deviation vectors $d^{+*}, d^{-*}$. Let $\hat{x}$ be any integer-feasible
rounding of $x^*_{\text{LP}}$ (i.e., $\hat{x}_i \in \mathbb{Z}$,
$\ell_i \leq \hat{x}_i \leq u_i$), and let $\delta_m = \sum_i c_{i,m}
(\hat{x}_i - x^*_{\text{LP},i})$ be the per-macro rounding perturbation.
Then there exist non-negative $\hat{d}^+, \hat{d}^-$ satisfying the goal
constraints~\eqref{eq:migp_goal} at $\hat{x}$, with
\[
\sum_{m} w_m (\hat{d}_m^+ + \hat{d}_m^-)
\;\leq\;
Z_{\text{LP}} + \sum_{m} w_m |\delta_m|.
\]
When $|\mathcal{F}| > |\mathcal{M}|$, the goal constraints
are underdetermined in $x$, giving the solver multiple integer roundings
among which it can minimize $\sum_m w_m |\delta_m|$. The deviation
variables then redistribute residual costs via the triangle inequality:
$|d_m^+ - d_m^-| \leq |d^{+*}_m - d^{-*}_m| + |\delta_m|$.
\end{proposition}

\begin{proof}
Given the goal constraint $\sum_i c_{i,m} \hat{x}_i + \hat{d}_m^- -
\hat{d}_m^+ = T_m$ and the LP constraint $\sum_i c_{i,m} x^*_i +
d^{-*}_m - d^{+*}_m = T_m$, subtracting yields $\hat{d}_m^- -
\hat{d}_m^+ = d^{-*}_m - d^{+*}_m - \delta_m$. Setting
$\hat{d}_m^+ = \max(0, d^{+*}_m - d^{-*}_m + \delta_m)$ and
$\hat{d}_m^- = \max(0, d^{-*}_m - d^{+*}_m - \delta_m)$ gives a
feasible assignment. By the triangle inequality,
$\hat{d}_m^+ + \hat{d}_m^- \leq (d^{+*}_m + d^{-*}_m) + |\delta_m|$.
Weighting and summing over $m$ yields the bound. The underdetermined
claim follows from $n > |\mathcal{M}|$: the nutrient matrix has rank
at most $|\mathcal{M}|$, so its null space is non-trivial, giving the
solver integer directions along which $\delta$ can be reduced.
\end{proof}

\paragraph{Contrast with Hard-Constraint IP.}
In hard-constraint IP, there are no deviation variables. Any rounding
perturbation that exceeds the tolerance band makes the point infeasible.
MIGP avoids this cliff effect: the deviation variables absorb rounding
at a finite, bounded cost instead of declaring infeasibility.
\label{rem:hard_ip_contrast}

When the food pool is large enough, the continuous relaxation already
produces near-integer solutions. The deviation variables absorb the
small remaining rounding costs, making the integer solution incur negligible additional cost.

\begin{remark}[Connection to Shapley--Folkman]
\label{rem:shapley_folkman}
The Shapley--Folkman theorem~\cite{starr1969quasi} implies that in a
problem with $|\mathcal{M}|=4$ goal constraints and $n$ integer
variables, at most $|\mathcal{M}|=4$ variables need non-trivial
rounding---the remaining $n - 4$ can take their continuous optima
rounded to the nearest feasible integer with negligible cost.
In the MIGP context, this means the solver only needs to ``search''
over a small subset of foods; the deviation variables absorb the
residual cost of those few rounding decisions. This explains the
empirical threshold observed in Section~\ref{sec:gap_char}: once
$n \geq 15 \gg 4$, the surplus degrees of freedom make the rounding
cost effectively zero.
\end{remark}

This property is not specific to diet optimization. Any goal
programming formulation with integer variables will exhibit deviation
absorption, making the integrality gap structurally smaller than in
hard-constraint MIP. The Shapley--Folkman connection
(Remark~\ref{rem:shapley_folkman}) further explains why: the number of
``difficult'' rounding decisions is bounded by the number of constraints,
not the number of variables.

\subsection{MIGP vs.\ Continuous GP with Rounding}
\label{sec:vs_rounding}

The standard practice in diet optimization is to solve a continuous LP
and round afterward~\cite{gazan2018mathematical}. We compare this
(GP+rounding: solve LP, round to nearest integer, clamp to bounds)
against MIGP, which enforces integrality during optimization.

Figure~\ref{fig:migp_rounding} compares objective values across all
270 instances. Points below the diagonal mean MIGP found a better
solution.

\begin{figure}[H]
\centering
\includegraphics[width=0.8\textwidth]{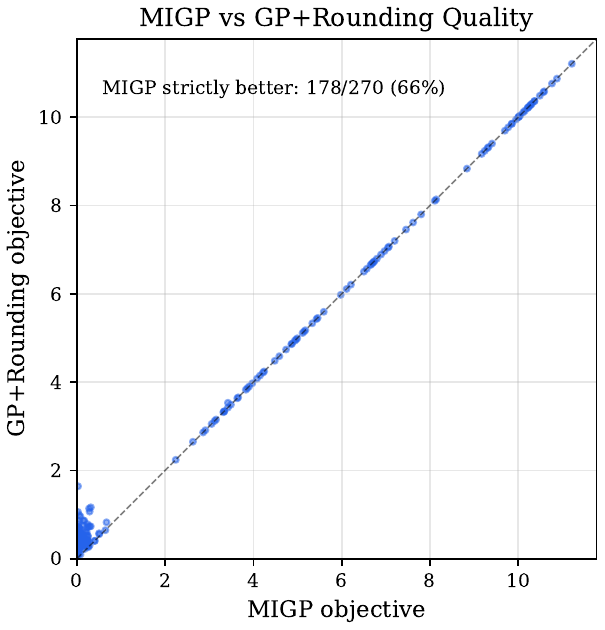}
\caption{MIGP objective vs.\ GP+rounding objective for all 270 instances.
Points below the diagonal (dashed) indicate MIGP superiority. MIGP finds
a strictly better solution in 66\% of instances and is never worse. The
cluster near the diagonal at high objective values corresponds to ambitious
configurations where both methods struggle equally.}
\label{fig:migp_rounding}
\end{figure}

The results separate cleanly by configuration type:

\paragraph{Non-ambitious (loose and tight bounds).}
MIGP is strictly better in \textbf{98\%} of instances (176/180) and
never worse. The improvement grows with size: median 61\% for 8~foods,
93\% for 15~foods, 98\% for 25~foods. With more foods, GP+rounding
increasingly rounds in conflicting directions (two foods up when only
one should be), while MIGP considers all rounding decisions jointly.

\paragraph{Ambitious (forced minimums, narrow range).}
Both methods produce essentially identical solutions in 98\% of
instances, with MIGP better in only 2\% and never worse. The targets
are inherently difficult here (often unachievable even continuously),
and the tight bounds leave little room for optimization. The LP already
produces near-integer solutions, so rounding introduces negligible cost.

\paragraph{Why rounding fails.}
Post-hoc rounding treats each food independently: round to the nearest
integer without considering the other macros. MIGP solves the rounding
problem globally. If the continuous optimum prescribes 2.3 and
1.7~servings, rounding gives (2,~2). But MIGP may find (3,~1)
produces lower total deviation because rounding one up compensates for
rounding the other down across multiple macros simultaneously.

\subsection{Practical Implications}
\label{sec:gap_practical}

For typical meal planning scenarios (5--20 foods, user-selected from a
database):

\paragraph{Practitioner guidelines.}
\begin{itemize}[leftmargin=*, itemsep=4pt]
\item \textbf{10+ foods: integrality is free.} The gap is zero in all
medium and large instances. A diverse food selection pays no penalty
for requiring integer servings.

\item \textbf{Under 10 foods: MIGP still wins.} Even when positive
gaps occur, MIGP beats GP+rounding in 98\% of non-ambitious instances
and is never worse. The gap measures distance from the LP lower bound,
not the quality relative to alternatives.

\item \textbf{Ambitious targets are the real challenge.} When targets
are inherently unachievable (e.g., very high protein with limited
protein sources), the issue is feasibility, not integrality. Absolute
deviation (Section~\ref{sec:experiments}) is a better metric here.

\item \textbf{Use continuous GP as a diagnostic.} Zero LP deviation
means the targets are achievable with fractions, so expect low integer
deviations. Positive LP deviation means no method can hit the targets
exactly; the user should adjust foods or targets.
\end{itemize}

%% file: sections/implementation.tex
\section{Implementation}
\label{sec:implementation}

The MIGP formulation is implemented as an open-source Python module for integration
into interactive meal planning applications. The code maps directly from the
mathematical formulation (Section~\ref{sec:formulation}); the integrality analysis
(Section~\ref{sec:integrality}) runs through the same solver.

\subsection{Solver Pipeline}

Algorithm~\ref{alg:pipeline} shows the end-to-end pipeline. Four steps: derive
targets, pre-check feasibility, build and solve the MILP, extract results.

\begin{algorithm}[H]
\DontPrintSemicolon
\KwIn{Foods $\mathcal{F}$ with nutrition $a_{i,m}$, serving sizes $s_i$, bounds $[\ell_i, u_i]$; target $T_{\text{cal}}, p_{\text{prot}}, p_{\text{carbs}}, p_{\text{fat}}$}
\KwOut{SolverResult with allocations, deviations, objective value}
\BlankLine
\tcp{Step 1: Derive gram targets}
$T_m \leftarrow$ derive\_targets($T_{\text{cal}}, p_{\text{prot}}, p_{\text{carbs}}, p_{\text{fat}}$)\;
\BlankLine
\tcp{Step 2: Feasibility pre-check}
\ForEach{$m \in \mathcal{M}$}{
  $\text{max}_m \leftarrow \sum_i c_{i,m} \cdot u_i$\;
  \If{$\text{max}_m < T_m$}{
    \Return \textsc{Infeasible}($m$, $T_m$, $\text{max}_m$)\;
  }
}
\BlankLine
\tcp{Step 3: Build and solve MILP}
$c \leftarrow [0^n \mid w_{\text{cal}}, \ldots, w_{\text{fat}} \mid w_{\text{cal}}, \ldots, w_{\text{fat}}]$\;
$A_{\text{eq}} \leftarrow$ goal constraint matrix \tcp*{$|\mathcal{M}| \times (n + 2|\mathcal{M}|)$}
$x^* \leftarrow \texttt{milp}(c, A_{\text{eq}}, T, \text{bounds}, \text{integrality})$\;
\BlankLine
\tcp{Step 4: Extract result}
$\text{allocations} \leftarrow \{(i, \lceil x^*_i \rfloor) : x^*_i > 0\}$\;
$\text{deviations} \leftarrow$ compute per-macro deviations from allocations\;
\Return SolverResult(allocations, deviations, $f(x^*)$)\;
\caption{MIGP Solver Pipeline}
\label{alg:pipeline}
\end{algorithm}

\subsection{Solver Backend}

The model solves via \texttt{scipy.optimize.milp}, which wraps the HiGHS
solver~\cite{huangfu2018highs} (open-source, high-performance, no commercial license
required).

The decision vector concatenates serving counts ($n$ integer variables), over-deviations
(4 continuous), and under-deviations (4 continuous): $n + 8$ variables total. An
integrality mask marks the first $n$ entries as integer and the remaining 8 as
continuous. The four goal constraints encode as a single equality
\texttt{LinearConstraint}; serving bounds go into a \texttt{Bounds} object.

The module also exposes a continuous relaxation (\texttt{optimize\_continuous}) that
solves the LP relaxation for integrality gap computation. Both methods share identical
problem construction; the only difference is the integrality mask.

\subsection{Computational Complexity}

With 4 macronutrients, the MILP has $n + 8$ variables ($n$ integer serving counts,
8 continuous deviations) and just 4 equality constraints. The continuous relaxation
is a small LP (polynomial time). The integer component is NP-hard in theory, but
meal planning instances are small enough that this never matters in practice.

Three structural properties make the branch-and-bound search efficient for MIGP.
First, only 4~equality constraints means the LP relaxation at each node is tiny
and solves in microseconds. Second, deviation absorption
(Proposition~\ref{prop:absorption}) produces tight LP bounds, enabling aggressive
pruning of suboptimal branches early in the search tree. Third, the
Shapley--Folkman argument (Remark~\ref{rem:shapley_folkman}) implies that at most
4~variables need non-trivial branching; the remaining variables can be fixed at
their LP-rounded values without loss of optimality.

Empirical scaling is polynomial in $n$ for our benchmark range: median solve times are 13\,ms
for 8 foods, 47\,ms for 15 foods, and 1.1\,s for 25 foods
(Section~\ref{sec:experiments}). A configurable time limit (default 30\,s) guards
against pathological instances; all benchmark instances solve well within it.

\subsection{Feasibility Pre-Check}

Before calling the solver, the pipeline computes the maximum achievable value of
each macro (all foods at maximum servings). If any macro target exceeds that maximum,
the instance is provably infeasible and the system returns an informative message
immediately: ``Cannot reach protein target of 80\,g. Maximum achievable: 62\,g.''
No need to wait for the solver to report the same thing.

\subsection{Data Source}

Nutritional data comes from FatSecret, a food database with per-100g macronutrient
values for over 500,000 foods. The application supports two search modes: local
database search over previously imported foods, and live web search that retrieves
and imports new foods on demand. All values are normalized to per-100g before
storage, so the linear scaling property $c_{i,m} = a_{i,m} \cdot s_i / 100$ holds
without additional conversion.

\subsection{User Interface}

An interactive Streamlit web application provides the user-facing interface. The
workflow follows natural meal planning order:
\begin{enumerate}
\item \textbf{Set targets}: calorie goal and macronutrient percentages via sliders
\item \textbf{Select foods}: search the database and add foods to the meal
\item \textbf{Configure servings}: adjust per-food serving size, minimum, and maximum bounds
\item \textbf{Optimize}: invoke the MIGP solver (sub-100\,ms for typical 8--15 food meals)
\item \textbf{Review results}: allocation table, macronutrient deviation report, and
  donut chart showing target vs.\ achieved composition
\end{enumerate}

The solver runs as a standalone module with no UI dependency. It accepts
Pydantic-validated \texttt{FoodInput} and \texttt{MealTarget} objects and returns a
structured \texttt{SolverResult}, making it trivial to wire into other interfaces
such as a Telegram bot or REST API.

\subsection{Reproducibility}

The complete implementation (solver, benchmark suite, food profiles, comparison
methods, and figure generation) is open-source Python. Running the benchmark script
reproduces all CSV data files and figures used in Section~\ref{sec:experiments}. The
solver has no external dependencies beyond SciPy and NumPy.

With the implementation established, we now evaluate the MIGP formulation empirically
against two alternative approaches across a benchmark of 810 instances.

%% file: sections/experiments.tex
\section{Computational Evaluation}
\label{sec:experiments}

\subsection{Benchmark Design}

We evaluate MIGP across a matrix of configurations designed to cover realistic meal planning scenarios.

\subsubsection{Food Bank}

All benchmark instances draw from a bank of 30 foods with per-100g nutritional values from the USDA FoodData Central database. Foods span four profiles: high-protein (chicken breast, salmon, eggs, tuna, tofu, beef, cottage cheese, Greek yogurt), high-carbohydrate (rice, oats, pasta, sweet potato, banana, bread, potato), high-fat (olive oil, almonds, peanut butter, avocado, cheese, walnuts), and balanced (lentils, chickpeas, milk, broccoli, quinoa, black beans, spinach, whey protein). Each food has a realistic default serving size (e.g., 50\,g per egg, 150\,g per rice portion, 15\,g per tablespoon of olive oil).

\subsubsection{Configurations}

Nine configurations combine three problem sizes with three difficulty levels (Table~\ref{tab:configs}). For each configuration, 30 instances are generated by sampling foods without replacement using distinct random seeds: 270 instances per method, 810 total benchmark runs.

\begin{table}[H]
\centering
\begin{tabular}{@{}l c c c c@{}}
\toprule
\textbf{Config} & \textbf{Foods} & \textbf{Serving range} & \textbf{Calories} & \textbf{Macro split} \\
\midrule
Small   & 8  & \multirow{3}{*}{Loose: $[0, 10]$}    & 600  & 30/45/25 \\
Medium  & 15 &                                        & 800  & 30/45/25 \\
Large   & 25 &                                        & 1000 & 30/45/25 \\
\midrule
Small   & 8  & \multirow{3}{*}{Tight: $[0, 4]$}      & 600  & 35/40/25 \\
Medium  & 15 &                                        & 800  & 35/40/25 \\
Large   & 25 &                                        & 1000 & 35/40/25 \\
\midrule
Small   & 8  & \multirow{3}{*}{Ambitious: $[1, 3]$}   & 600  & 40/35/25 \\
Medium  & 15 &                                        & 800  & 40/35/25 \\
Large   & 25 &                                        & 1000 & 40/35/25 \\
\bottomrule
\end{tabular}
\caption{Benchmark configurations. Macro split is protein/carbs/fat
percentage. Ambitious configs force a minimum of 1~serving per food,
creating challenging targets.}
\label{tab:configs}
\end{table}

The \emph{ambitious} configuration is worth noting: by requiring every food to contribute at least 1~serving, the model must absorb nutrients from all selected foods, often making the original macro targets unachievable. This simulates scenarios where dietary variety is imposed as a hard constraint.

\subsection{Comparison Methods}

Three methods are compared, all using the same food inputs and targets:

\paragraph{MIGP (ours).} The full formulation from Section~\ref{sec:formulation},
solved with \texttt{scipy.optimize.milp} (HiGHS backend). Integer serving
variables with goal programming deviation minimization.

\paragraph{GP+Rounding.} The continuous relaxation of MIGP is solved
(standard LP), then each fractional serving is rounded to the nearest
integer and clamped to the serving bounds. This represents the current
practice documented by Gazan et al.~\cite{gazan2018mathematical}, where
continuous diet optimization is followed by post-hoc discretization.

\paragraph{Hard-Constraint IP.} A standard integer program that minimizes
total servings subject to hard $\pm$5\% tolerance bands on each macronutrient.
This represents the integer programming approach to diet
optimization~\cite{benvenuti2024designing}, where nutrient requirements
are enforced as inequality constraints rather than soft goals.

\subsection{Metrics}

For each instance, we record:
\begin{itemize}
\item \textbf{Feasibility}: whether the method returns a valid solution
\item \textbf{Objective value}: weighted sum of absolute deviations
  (Equation~\ref{eq:objective}), computed identically for all methods to
  enable fair comparison
\item \textbf{Max deviation \%}: largest single-macro deviation as a
  percentage of its target
\item \textbf{Macros $\leq$5\%}: percentage of macronutrients within 5\%
  of their target (practical quality threshold)
\item \textbf{Solve time}: wall-clock milliseconds including solver
  overhead
\end{itemize}

\subsection{Results}

\subsubsection{Method Comparison}

Table~\ref{tab:comparison} summarizes performance across all 270 instances
per method.

\input{figures/comparison_table}

\paragraph{Feasibility.}
MIGP achieves 100\% feasibility across all configurations, a direct
consequence of Proposition~\ref{prop:feasibility}. GP+Rounding also
achieves 100\% feasibility, since rounding a feasible LP solution always
produces an integer point within the serving bounds. Hard-Constraint IP
fails on 140 of 270 instances (51.9\%), including all 90 ambitious instances
and the majority of small instances.

Figure~\ref{fig:feasibility} shows feasibility rates by constraint
tightness. The gap between MIGP and Hard-IP widens as constraints tighten:
Hard-IP achieves 74\% feasibility with loose bounds but 0\% with ambitious bounds.

\begin{figure}[H]
\centering
\includegraphics[width=0.8\textwidth]{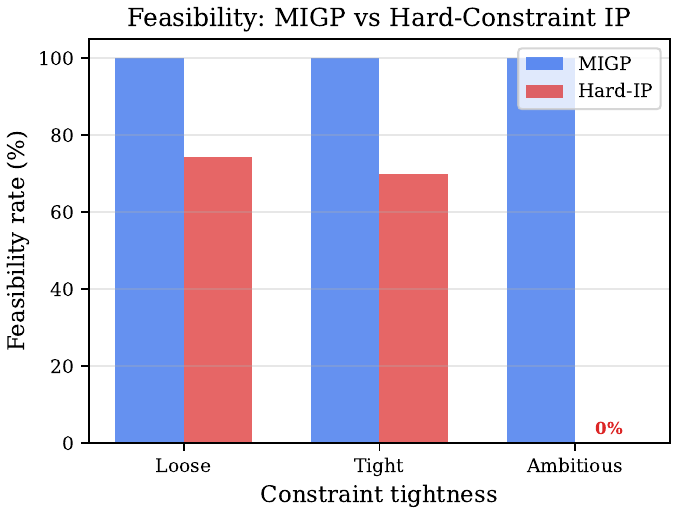}
\caption{Feasibility rates by constraint tightness (Y-axis capped at
105\% for readability). MIGP maintains 100\% feasibility regardless of
configuration. Hard-Constraint IP degrades as targets become more
restrictive, failing entirely on ambitious instances where every food
must contribute at least 1~serving. GP+Rounding (not shown) also achieves
100\% feasibility since rounding preserves serving bounds.}
\label{fig:feasibility}
\end{figure}

\paragraph{Solution quality.}
Among feasible solutions, MIGP achieves a median objective of 0.141 with
a median maximum deviation of 6.3\%. For non-ambitious configurations, the
results are substantially better: median maximum deviation is 1.4\% for
medium and 0.3\% for large instances, with 100\% of macros within the 5\%
practical threshold. GP+Rounding produces a median objective 3.8$\times$
higher than MIGP (0.529 vs.\ 0.141), with median maximum deviation of
21.6\%.

Hard-Constraint IP, when feasible, produces low-deviation solutions by
construction (the $\pm$5\% tolerance band ensures tight adherence). But it
provides \emph{no solution at all} for over half of instances, making
it unreliable for interactive use.

Across 810 benchmark instances, MIGP achieves 100\% feasibility
while finding strictly better solutions than GP+Rounding in 66\% of
instances (never worse). Hard-Constraint IP fails on 51.9\% of instances
and returns no solution at all for any ambitious configuration.

\subsubsection{Solve Time}

Figure~\ref{fig:solve_time} shows solve time scaling with problem size
(log scale). MIGP is the slowest method due to the integer programming
component, with median times of 13\,ms (8~foods), 47\,ms (15~foods),
and 1.1\,s (25~foods). GP+Rounding is consistently fast at $\sim$1\,ms
regardless of size, since it solves only an LP. Hard-Constraint IP
ranges from 1--5\,ms, benefiting from a simpler objective (no deviation
variables).

\begin{figure}[H]
\centering
\includegraphics[width=0.8\textwidth]{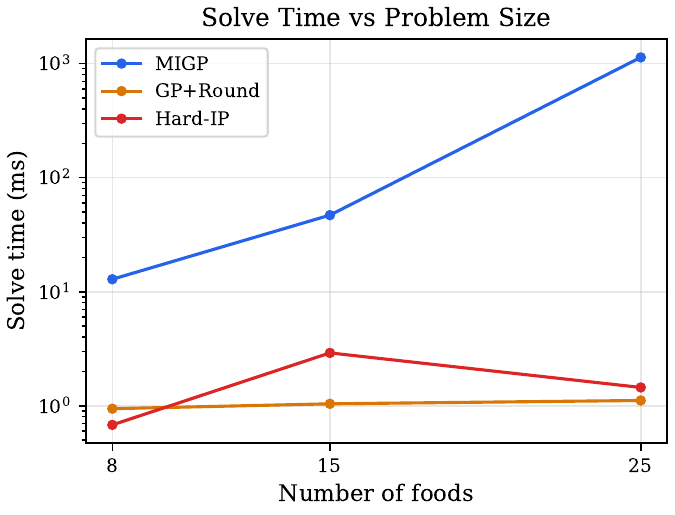}
\caption{Median solve time vs.\ number of foods (log scale). MIGP scales
from 13\,ms to 1.1\,s as problem size increases. GP+Rounding and Hard-IP
remain under 5\,ms. All methods are fast enough for interactive use at
typical meal sizes (8--15 foods).}
\label{fig:solve_time}
\end{figure}

For typical interactive meal planning (8--15 foods), MIGP solves in under
100\,ms, imperceptible to the user. The 25-food case (1.1\,s median)
is an upper bound that rarely occurs in practice; users typically select
5--15 foods for a single meal.

\subsubsection{Deviation Distribution}

Figure~\ref{fig:deviation_boxplots} shows the distribution of maximum
macro deviation by problem size across all three methods. The box plots
include all configurations (loose, tight, and ambitious), so the wide
ranges reflect ambitious instances with inherently unachievable targets.

\begin{figure}[H]
\centering
\includegraphics[width=0.85\textwidth]{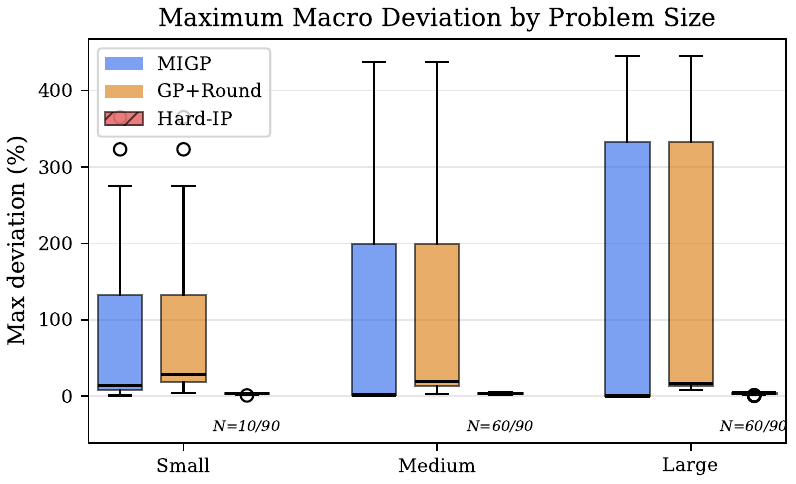}
\caption{Maximum macro deviation by problem size and method. Hard-IP
results shown only for feasible instances ($N$ shown below each box);
the 140 infeasible instances (51.9\%) are excluded, creating a strong
selection bias toward easy cases. Hard-IP's apparent low deviations
reflect its $\pm$5\% tolerance band, not superior optimization. Wide
ranges for MIGP and GP+Rounding are driven by ambitious configurations
where targets are inherently unachievable.}
\label{fig:deviation_boxplots}
\end{figure}

Hard-IP's apparent superiority in deviation is misleading: it achieves
low deviations precisely because it rejects instances where targets cannot
be met within 5\%. The 140 instances where Hard-IP fails are excluded from
the box plot. MIGP provides a solution for every instance, accepting
higher deviations on difficult problems rather than returning nothing.

\subsubsection{Hard-IP Failure Analysis}

The 51.9\% infeasibility rate of Hard-Constraint IP is not uniformly
distributed. All 90 ambitious instances fail (100\%), since the forced
minimum of 1 serving per food makes the $\pm$5\% tolerance bands impossible
to satisfy simultaneously across all four macros. Among non-ambitious
configurations, small instances (8~foods) suffer the most: fewer decision
variables means the solver has insufficient degrees of freedom to satisfy
all hard constraints. Medium and large loose instances achieve near-perfect
feasibility because more foods provide more combinations that fall within the
tolerance bands.

This pattern reveals the fragility of hard-constraint formulations: feasibility
depends not only on the target but on the combinatorial structure of the
available foods, something users cannot control or predict.

\subsection{Worked Examples}
\label{sec:examples}

The benchmark results aggregate over hundreds of random instances. To build
intuition for how the three methods behave on concrete meal planning scenarios,
we present five representative examples spanning single meals, batch recipes, and
high-protein recovery scenarios. All examples use foods from the USDA-sourced
food bank. Bulk whole foods (chicken, rice, broccoli) use 50\,g serving sizes
to provide fine-grained optimization, while packaged or naturally portioned
items (eggs, olive oil, protein scoops, canned tuna) retain their natural
serving units. This dual approach (small servings for measurable bulk foods,
natural units for discrete items) gives the optimizer more granularity where it
matters most, consistent with the granularity analysis in
Section~\ref{sec:sensitivity}.

\subsubsection{Example A: Post-Workout Recovery Meal}

A simple scenario with 5~foods targeting 600\,kcal at a 30/45/25
protein/carbs/fat split. All whole foods use 50\,g serving sizes
for fine-grained optimization; avocado uses 30\,g and olive oil
retains its natural 15\,g tablespoon. We begin by deriving the gram
targets from the percentage-based specification
(Table~\ref{tab:example_a_targets}), then examine the food pool
(Table~\ref{tab:example_a_input}).

\input{figures/example_a_targets}
\input{figures/example_a_input}

Table~\ref{tab:example_a_breakdown} shows the MIGP solution with a
per-food nutrient breakdown and running totals. MIGP selects four of the
five foods: 2~servings of chicken breast (100\,g), 3~servings of rice
(150\,g), 5~servings of broccoli (250\,g), and 3~servings of avocado
(90\,g). The maximum deviation is 9.7\% on fat, reflecting that even
with fine granularity, the fat sources available (avocado at 4.4\,g per
serving, olive oil at 15\,g) do not divide evenly into the 16.7\,g target.

\input{figures/example_a_breakdown}

To understand why MIGP outperforms GP+Rounding, Table~\ref{tab:example_a_sidebyside}
compares the continuous LP solution, the rounded allocation, and the MIGP
integer allocation side by side. The LP relaxation finds a fractional
optimum (e.g., 2.33~servings of chicken, 2.62~servings of avocado)
that rounds to a different integer solution than MIGP's joint optimization.

\input{figures/example_a_sidebyside}
\input{figures/example_a_results}

\input{figures/example_a_commentary}

Both MIGP and GP+Rounding achieve similar maximum single-macro deviations
(9.7\% vs.\ 9.8\%), but MIGP produces a better overall objective
(0.165 vs.\ 0.230) because it distributes deviations more efficiently
across macros. Hard-IP fails entirely: with only 5~foods, the $\pm$5\%
tolerance bands on all four macros simultaneously are too restrictive.

\subsubsection{Example B: Balanced Lunch}

A medium scenario with 8~foods targeting 800\,kcal at 35/40/25.
All bulk foods use 50\,g serving sizes; olive oil (15\,g), almonds
(30\,g), and eggs (50\,g) retain their natural portion units.
Table~\ref{tab:example_b_targets} derives the gram targets, and
Table~\ref{tab:example_b_input} lists the food pool.

\input{figures/example_b_targets}
\input{figures/example_b_input}

This example shows \emph{deviation absorption} in action.
Table~\ref{tab:example_b_breakdown} shows that MIGP achieves an
objective of 0.051 with a maximum deviation of only 2.4\%, selecting
4 of the 8 available foods. The solver concentrates servings on
chicken breast (200\,g), rice (100\,g), sweet potato (250\,g),
and olive oil (15\,g), the combination that best matches the macro
profile, rather than spreading thinly across all options.

\input{figures/example_b_breakdown}

Table~\ref{tab:example_b_sidebyside} shows why: the LP relaxation
distributes servings across 4~foods (including fractional amounts of
broccoli and eggs), but rounding each independently accumulates errors
that cannot be redistributed.

\input{figures/example_b_sidebyside}
\input{figures/example_b_results}

GP+Rounding produces a 2.8$\times$ worse objective (0.145) from this
accumulated rounding error. Hard-IP succeeds here (8~foods with
50\,g granularity provide enough degrees of freedom), producing a
reasonable solution with 4.1\% maximum deviation---within its tolerance
band but still 1.8$\times$ the MIGP objective.

\subsubsection{Example C: Variety Dinner (Ambitious Bounds)}

An ambitious scenario forcing all 8~foods to contribute at least 1~serving,
targeting 600\,kcal at 40/35/25. Serving sizes remain small (50\,g for
bulk foods, 30\,g for avocado, 15\,g for olive oil), but the minimum of
1~serving per food forces the model to include every item.
Table~\ref{tab:example_c_targets} derives the gram targets, and
Table~\ref{tab:example_c_input} lists the food pool with its constrained
serving ranges.

\input{figures/example_c_targets}
\input{figures/example_c_input}

This example shows MIGP's graceful degradation.
Table~\ref{tab:example_c_breakdown} shows the solution: even with small
serving sizes, forcing all 8~foods to contribute produces a meal
exceeding the 600\,kcal target by 28\%. Fat is the most affected
macro (121\% over target), driven by the mandatory olive oil, avocado,
and egg servings.

\input{figures/example_c_breakdown}

Table~\ref{tab:example_c_sidebyside} compares the three allocation
strategies. The LP relaxation is only slightly fractional (chicken at
2.06~servings, rice at 1.22), so rounding produces a nearly identical
allocation. MIGP adjusts rice from 1 to 2~servings to reduce broccoli
from 6 to 3~servings, trading carb accuracy for a lower overall objective.

\input{figures/example_c_sidebyside}
\input{figures/example_c_results}

MIGP produces the best-effort solution with 28\% calorie deviation.
Hard-IP fails because no integer assignment can satisfy the $\pm$5\%
tolerance with mandatory minimums. The key distinction is behavioral:
MIGP and GP+Rounding \emph{return a meal} that a user can adjust, while
Hard-IP returns nothing. With small serving sizes, MIGP achieves a tighter
solution (obj 1.556) than GP+Rounding (1.578)---joint integer optimization
still helps even in over-constrained regimes.

\subsubsection{Example D: Cyclist Energy Snack Batch}

A batch recipe scenario: 1000\,kcal total (to be divided into 10
$\times$ 100\,kcal portions), with an extreme 10/75/15 (P/C/F) split
typical of endurance cycling nutrition. Six foods, including three
carb-dense ingredients (honey at 15\,g, dates at 30\,g, oats at
40\,g) alongside banana, peanut butter, and dark chocolate chips.
All serving bounds are [0, $n$] with no forced minimums.

This example demonstrates MIGP as a \emph{recipe optimization tool}:
the solver finds total ingredient quantities for a 1000\,kcal batch,
and the user divides into portions. The extreme macro split (75\%
carbohydrates, only 10\% protein) tests the model with a heavily skewed
target. With 6 carb-rich foods available, deviation absorption should
handle the extreme split effectively.

\input{figures/example_d_targets}
\input{figures/example_d_input}
\input{figures/example_d_breakdown}
\input{figures/example_d_sidebyside}
\input{figures/example_d_results}

\input{figures/example_d_commentary}

\subsubsection{Example E: Post-Gym Protein Recovery}

A post-gym recovery scenario: 600\,kcal with a high-protein 45/30/25
(P/C/F) split. Six foods including two ``packaged'' items: whey protein
powder (30\,g scoop, minimum 1~scoop forced) and canned tuna (80\,g
can). Olive oil uses a 5\,g teaspoon serving for fine-grained fat
adjustment.

This example demonstrates MIGP with foods whose serving sizes are
dictated by their packaging or physical form. The protein scoop
and tuna can are natural discrete units that continuous optimization
would split into impractical fractions (e.g., 1.4~scoops of protein
powder). The minimum 1-scoop constraint shows forced inclusion, and
the high protein target (45\%) tests the model's ability to concentrate
on protein-dense foods while balancing the remaining macros.

\input{figures/example_e_targets}
\input{figures/example_e_input}
\input{figures/example_e_breakdown}
\input{figures/example_e_sidebyside}
\input{figures/example_e_results}

\input{figures/example_e_commentary}

\subsection{Sensitivity Analysis}
\label{sec:sensitivity}

We evaluate sensitivity to two design choices using the medium-loose
configuration (15~foods, $[0, 10]$ serving range, 800\,kcal) as a
representative scenario.

\subsubsection{Penalty Weight Schemes}

Three weighting strategies are compared (Figure~\ref{fig:sensitivity}):
\begin{itemize}
\item \textbf{Inverse-target} (default): $w_m = 1/\max(T_m, 1)$
\item \textbf{Equal weights}: $w_m = 1$ for all macros
\item \textbf{Double protein}: inverse-target with $2\times$ protein weight
\end{itemize}

\begin{figure}[H]
\centering
\includegraphics[width=0.8\textwidth]{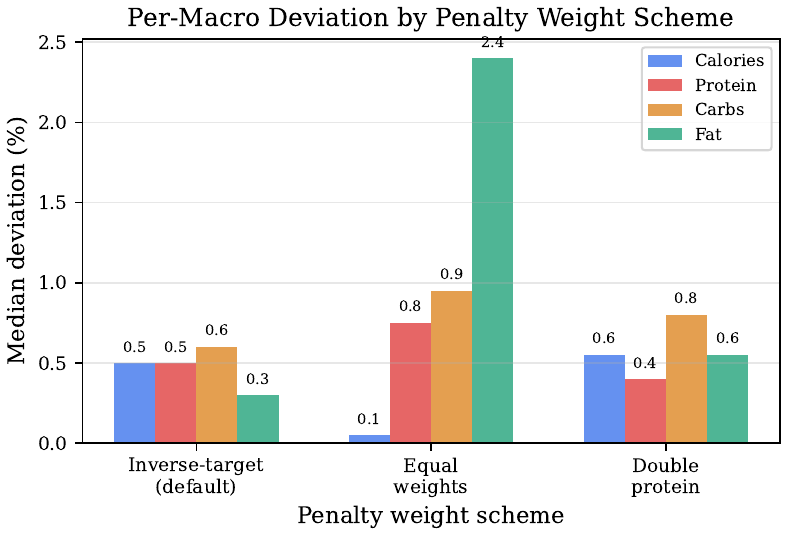}
\caption{Sensitivity to penalty weight scheme (medium-loose config, 30
instances). Objective values across schemes are \emph{not directly
comparable}: each scheme uses different penalty magnitudes $w_m$, so a
lower objective under equal weights does not imply a better solution than
under inverse-target weights. The meaningful comparison is in per-macro
deviation percentages: inverse-target and double-protein achieve balanced
deviations (all median per-macro deviations under 1\%); equal weights
allow fat to deviate up to 2.4\% (median).}
\label{fig:sensitivity}
\end{figure}

Objective values across schemes are not directly comparable, since different
weight magnitudes yield different scales. What matters is macro-level
deviation: inverse-target achieves balanced deviations across all macros
(median per-macro deviations under 1\%), while equal weights allow fat to
deviate up to 2.4\% (median) because calorie deviations, with larger
absolute magnitude, dominate the objective. The double-protein scheme
successfully reduces protein deviation to 0.4\% (vs.\ 0.5\% for
inverse-target) at negligible cost to other macros, demonstrating that
the penalty weight mechanism provides effective user control over macro
prioritization.

\subsubsection{Serving Granularity}

Serving size determines the discretization granularity of the solution
space. We test four uniform serving sizes (25\,g, 50\,g, 100\,g, 200\,g)
on the same food instances (Table~\ref{tab:granularity}).

\begin{table}[H]
\centering
\begin{tabular}{@{}l c c c c@{}}
\toprule
\textbf{Serving size} & \textbf{Median obj.} & \textbf{Max dev.\ \%} & \textbf{Within 5\%} & \textbf{Median time} \\
\midrule
25\,g  & 0.001 & 0.1\% & 100\% & 2651\,ms \\
50\,g  & 0.008 & 0.3\% & 100\% & 131\,ms \\
100\,g & 0.039 & 1.9\% & 100\% & 50\,ms \\
200\,g & 0.108 & 6.2\% & 75\%  & 30\,ms \\
\bottomrule
\end{tabular}
\caption{Effect of serving granularity on solution quality and solve
time. Finer granularity improves accuracy but increases computational
cost. The 100\,g default offers a practical tradeoff.}
\label{tab:granularity}
\end{table}

Finer granularity yields better solutions at higher cost. At 25\,g servings,
deviations are negligible (max 0.1\%) but solve time reaches 2.7\,s, nearly
two orders of magnitude slower than 200\,g servings. This tradeoff is not
just computational: 25\,g servings are impractical for real meal preparation,
since measuring 25\,g of chicken breast or 25\,g of rice requires a precision
scale and produces portions that do not correspond to natural serving units. The
default per-food serving sizes used in our implementation (ranging from 15\,g
for olive oil to 250\,g for milk) balance accuracy and efficiency by matching
serving granularity to the food's natural portion size.

As serving size decreases, the LP relaxation and integer solution converge
because finer granularity reduces the rounding gap. At 25\,g, the
integrality gap is effectively zero (median objective 0.001). At 200\,g,
the gap is substantial (median 0.108), and deviation absorption becomes
critical for producing usable solutions. The MIGP formulation's value is
greatest precisely in the practical regime where servings are
coarse-grained.

%% file: figures/comparison_table.tex
\begin{table}[t]
    \centering
    \caption{Comparison of optimization methods across all benchmark configurations.}
    \label{tab:comparison}
    \begin{tabular}{lccccc}
        \toprule
        Method & Feasibility & Median Obj. & Median Max Dev. & Macros $\leq$5\% & Median Time \\
        \midrule
        MIGP & 100\% & 0.1410 & 6.3\% & 75\% & 19.9 \\
        GP+Round & 100\% & 0.5286 & 21.6\% & 0\% & 1.0 \\
        Hard-IP & 48\% & 0.0911 & 4.0\% & 100\% & 3.5 \\
        \bottomrule
    \end{tabular}
\end{table}

%% file: figures/example_a_targets.tex
\begin{table}[H]
\centering
\caption{Example A: target derivation from 600\,kcal with 30/45/25 (P/C/F\%).}
\label{tab:example_a_targets}
\begin{tabular}{@{}l c c l r@{}}
    \toprule
    Macro & \% split & kcal/g & Calculation & Target \\
    \midrule
        Calories & --- & --- & $600$ & 600.0\,kcal \\
        Protein & 30\% & 4 & $600 \times 0.30 \;/\; 4 = 45.0$ & 45.0\,g \\
        Carbs & 45\% & 4 & $600 \times 0.45 \;/\; 4 = 67.5$ & 67.5\,g \\
        Fat & 25\% & 9 & $600 \times 0.25 \;/\; 9 = 16.7$ & 16.7\,g \\
    \bottomrule
\end{tabular}
\end{table}

%% file: figures/example_a_input.tex
\begin{table}[H]
\centering
\small
\caption{Example A: food inputs. Target: 600\,kcal, 30/45/25 (P/C/F\%). Derived gram targets: 45.0g protein, 67.5g carbs, 16.7g fat.}
\label{tab:example_a_input}
\begin{tabular}{@{}l c c c c c c c c c c@{}}
    \toprule
    & & & \multicolumn{4}{c}{Per 100\,g} & \multicolumn{4}{c}{Per serving} \\
    \cmidrule(lr){4-7} \cmidrule(lr){8-11}
    Food & Srv (g) & Range & kcal & P & C & F & kcal & P & C & F \\
    \midrule
        Chicken breast & 50 & [0, 10] & 165 & 31.0 & 0.0 & 3.6 & 82.5 & 15.5 & 0.0 & 1.8 \\
        White rice & 50 & [0, 10] & 130 & 2.7 & 28.2 & 0.3 & 65.0 & 1.4 & 14.1 & 0.1 \\
        Broccoli & 50 & [0, 8] & 35 & 2.4 & 7.2 & 0.4 & 17.5 & 1.2 & 3.6 & 0.2 \\
        Avocado & 30 & [0, 6] & 160 & 2.0 & 8.5 & 14.7 & 48.0 & 0.6 & 2.5 & 4.4 \\
        Olive oil & 15 & [0, 4] & 884 & 0.0 & 0.0 & 100.0 & 132.6 & 0.0 & 0.0 & 15.0 \\
    \bottomrule
\end{tabular}
\end{table}

%% file: figures/example_a_breakdown.tex
\begin{table}[H]
\centering
\caption{Example A: MIGP solution breakdown. Each row shows one food's contribution; the deviation row shows (actual $-$ target) / target as a percentage.}
\label{tab:example_a_breakdown}
\begin{tabular}{@{}l c r r r r@{}}
    \toprule
    Food (serving) & Qty & kcal & P (g) & C (g) & F (g) \\
    \midrule
        Chicken breast (50\,g) & 2 & 165.0 & 31.0 & 0.0 & 3.6 \\
        White rice (50\,g) & 3 & 195.0 & 4.0 & 42.3 & 0.5 \\
        Broccoli (50\,g) & 5 & 87.5 & 6.0 & 18.0 & 1.0 \\
        Avocado (30\,g) & 3 & 144.0 & 1.8 & 7.7 & 13.2 \\
        \midrule
        \textbf{Total} & & 591.5 & 42.9 & 68.0 & 18.3 \\
        \textbf{Target} & & 600.0 & 45.0 & 67.5 & 16.7 \\
        \textbf{Deviation} & & -1.4\% & -4.8\% & +0.7\% & +9.7\% \\
    \bottomrule
\end{tabular}
\end{table}

%% file: figures/example_a_sidebyside.tex
\begin{table}[H]
\centering
\caption{Example A: side-by-side method allocations. LP~Continuous shows fractional servings (lower bound $z_{LP}$); GP+Round rounds each food to the nearest integer; MIGP finds the globally optimal integer solution.}
\label{tab:example_a_sidebyside}
\begin{tabular}{@{}l r r r@{}}
    \toprule
    & LP Continuous & GP+Round & MIGP \\
    \midrule
        Chicken breast & 2.33 & 2 & 2 \\
        White rice & 3.88 & 4 & 3 \\
        Broccoli & 1.71 & 2 & 5 \\
        Avocado & 2.62 & 3 & 3 \\
        \midrule
        kcal & 600.0 (+0.0\%) & 604.0 (+0.7\%) & 591.5 (-1.4\%) \\
        P (g) & 45.0 (+0.0\%) & 40.6 (-9.8\%) & 42.9 (-4.8\%) \\
        C (g) & 67.5 (+0.0\%) & 71.2 (+5.6\%) & 68.0 (+0.7\%) \\
        F (g) & 16.7 (+0.0\%) & 17.8 (+7.0\%) & 18.3 (+9.7\%) \\
        \midrule
        Objective ($z$) & 0.0000 & 0.2298 & 0.1654 \\
    \bottomrule
\end{tabular}
\end{table}

%% file: figures/example_a_results.tex
\begin{table}[H]
\centering
\caption{Example A: method comparison with per-macro deviations (\%).}
\label{tab:example_a_compare}
\begin{tabular}{@{}l c c c c c c c@{}}
    \toprule
    Method & Feasible & Obj. & Max dev. & Cal & Prot & Carbs & Fat \\
    \midrule
        MIGP & Yes & 0.1654 & +9.7\% & -1.4 & -4.8 & +0.7 & +9.7 \\
        GP+Round & Yes & 0.2298 & 9.8\% & 0.7 & 9.8 & 5.6 & 7.0 \\
        Hard-IP & No & --- & --- & --- & --- & --- & --- \\
    \bottomrule
\end{tabular}
\end{table}

%% file: figures/example_a_commentary.tex

%% file: figures/example_b_targets.tex
\begin{table}[H]
\centering
\caption{Example B: target derivation from 800\,kcal with 35/40/25 (P/C/F\%).}
\label{tab:example_b_targets}
\begin{tabular}{@{}l c c l r@{}}
    \toprule
    Macro & \% split & kcal/g & Calculation & Target \\
    \midrule
        Calories & --- & --- & $800$ & 800.0\,kcal \\
        Protein & 35\% & 4 & $800 \times 0.35 \;/\; 4 = 70.0$ & 70.0\,g \\
        Carbs & 40\% & 4 & $800 \times 0.40 \;/\; 4 = 80.0$ & 80.0\,g \\
        Fat & 25\% & 9 & $800 \times 0.25 \;/\; 9 = 22.2$ & 22.2\,g \\
    \bottomrule
\end{tabular}
\end{table}

%% file: figures/example_b_input.tex
\begin{table}[H]
\centering
\small
\caption{Example B: food inputs. Target: 800\,kcal, 35/40/25 (P/C/F\%). Derived gram targets: 70.0g protein, 80.0g carbs, 22.2g fat.}
\label{tab:example_b_input}
\begin{tabular}{@{}l c c c c c c c c c c@{}}
    \toprule
    & & & \multicolumn{4}{c}{Per 100\,g} & \multicolumn{4}{c}{Per serving} \\
    \cmidrule(lr){4-7} \cmidrule(lr){8-11}
    Food & Srv (g) & Range & kcal & P & C & F & kcal & P & C & F \\
    \midrule
        Chicken breast & 50 & [0, 8] & 165 & 31.0 & 0.0 & 3.6 & 82.5 & 15.5 & 0.0 & 1.8 \\
        Salmon fillet & 50 & [0, 8] & 208 & 20.4 & 0.0 & 13.4 & 104.0 & 10.2 & 0.0 & 6.7 \\
        White rice & 50 & [0, 8] & 130 & 2.7 & 28.2 & 0.3 & 65.0 & 1.4 & 14.1 & 0.1 \\
        Sweet potato & 50 & [0, 8] & 90 & 2.0 & 20.7 & 0.1 & 45.0 & 1.0 & 10.3 & 0.1 \\
        Broccoli & 50 & [0, 8] & 35 & 2.4 & 7.2 & 0.4 & 17.5 & 1.2 & 3.6 & 0.2 \\
        Olive oil & 15 & [0, 4] & 884 & 0.0 & 0.0 & 100.0 & 132.6 & 0.0 & 0.0 & 15.0 \\
        Almonds & 30 & [0, 4] & 579 & 21.2 & 21.6 & 49.9 & 173.7 & 6.4 & 6.5 & 15.0 \\
        Whole eggs & 50 & [0, 4] & 155 & 12.6 & 1.1 & 10.6 & 77.5 & 6.3 & 0.6 & 5.3 \\
    \bottomrule
\end{tabular}
\end{table}

%% file: figures/example_b_breakdown.tex
\begin{table}[H]
\centering
\caption{Example B: MIGP solution breakdown. Each row shows one food's contribution; the deviation row shows (actual $-$ target) / target as a percentage.}
\label{tab:example_b_breakdown}
\begin{tabular}{@{}l c r r r r@{}}
    \toprule
    Food (serving) & Qty & kcal & P (g) & C (g) & F (g) \\
    \midrule
        Chicken breast (50\,g) & 4 & 330.0 & 62.0 & 0.0 & 7.2 \\
        White rice (50\,g) & 2 & 130.0 & 2.7 & 28.2 & 0.3 \\
        Sweet potato (50\,g) & 5 & 225.0 & 5.0 & 51.8 & 0.2 \\
        Olive oil (15\,g) & 1 & 132.6 & 0.0 & 0.0 & 15.0 \\
        \midrule
        \textbf{Total} & & 817.6 & 69.7 & 80.0 & 22.8 \\
        \textbf{Target} & & 800.0 & 70.0 & 80.0 & 22.2 \\
        \textbf{Deviation} & & +2.2\% & -0.4\% & -0.1\% & +2.4\% \\
    \bottomrule
\end{tabular}
\end{table}

%% file: figures/example_b_sidebyside.tex
\begin{table}[H]
\centering
\caption{Example B: side-by-side method allocations. LP~Continuous shows fractional servings (lower bound $z_{LP}$); GP+Round rounds each food to the nearest integer; MIGP finds the globally optimal integer solution.}
\label{tab:example_b_sidebyside}
\begin{tabular}{@{}l r r r@{}}
    \toprule
    & LP Continuous & GP+Round & MIGP \\
    \midrule
        Chicken breast & 2.55 & 3 & 4 \\
        White rice & 0.00 & 0 & 2 \\
        Sweet potato & 6.26 & 6 & 5 \\
        Broccoli & 3.76 & 4 & 0 \\
        Olive oil & 0.00 & 0 & 1 \\
        Whole eggs & 3.13 & 3 & 0 \\
        \midrule
        kcal & 800.0 (+0.0\%) & 820.0 (+2.5\%) & 817.6 (+2.2\%) \\
        P (g) & 70.0 (+0.0\%) & 76.2 (+8.9\%) & 69.7 (-0.4\%) \\
        C (g) & 80.0 (+0.0\%) & 78.2 (-2.3\%) & 80.0 (-0.1\%) \\
        F (g) & 22.2 (-0.0\%) & 22.4 (+0.8\%) & 22.8 (+2.4\%) \\
        \midrule
        Objective ($z$) & 0.0000 & 0.1447 & 0.0507 \\
    \bottomrule
\end{tabular}
\end{table}

%% file: figures/example_b_results.tex
\begin{table}[H]
\centering
\caption{Example B: method comparison with per-macro deviations (\%).}
\label{tab:example_b_compare}
\begin{tabular}{@{}l c c c c c c c@{}}
    \toprule
    Method & Feasible & Obj. & Max dev. & Cal & Prot & Carbs & Fat \\
    \midrule
        MIGP & Yes & 0.0507 & +2.4\% & +2.2 & -0.4 & -0.1 & +2.4 \\
        GP+Round & Yes & 0.1447 & 8.9\% & 2.5 & 8.9 & 2.3 & 0.8 \\
        Hard-IP & Yes & 0.0899 & 4.1\% & 4.1 & 0.4 & 1.1 & 3.5 \\
    \bottomrule
\end{tabular}
\end{table}

%% file: figures/example_c_targets.tex
\begin{table}[H]
\centering
\caption{Example C: target derivation from 600\,kcal with 40/35/25 (P/C/F\%).}
\label{tab:example_c_targets}
\begin{tabular}{@{}l c c l r@{}}
    \toprule
    Macro & \% split & kcal/g & Calculation & Target \\
    \midrule
        Calories & --- & --- & $600$ & 600.0\,kcal \\
        Protein & 40\% & 4 & $600 \times 0.40 \;/\; 4 = 60.0$ & 60.0\,g \\
        Carbs & 35\% & 4 & $600 \times 0.35 \;/\; 4 = 52.5$ & 52.5\,g \\
        Fat & 25\% & 9 & $600 \times 0.25 \;/\; 9 = 16.7$ & 16.7\,g \\
    \bottomrule
\end{tabular}
\end{table}

%% file: figures/example_c_input.tex
\begin{table}[H]
\centering
\small
\caption{Example C: food inputs. Target: 600\,kcal, 40/35/25 (P/C/F\%). Derived gram targets: 60.0g protein, 52.5g carbs, 16.7g fat.}
\label{tab:example_c_input}
\begin{tabular}{@{}l c c c c c c c c c c@{}}
    \toprule
    & & & \multicolumn{4}{c}{Per 100\,g} & \multicolumn{4}{c}{Per serving} \\
    \cmidrule(lr){4-7} \cmidrule(lr){8-11}
    Food & Srv (g) & Range & kcal & P & C & F & kcal & P & C & F \\
    \midrule
        Chicken breast & 50 & [1, 6] & 165 & 31.0 & 0.0 & 3.6 & 82.5 & 15.5 & 0.0 & 1.8 \\
        Salmon fillet & 50 & [1, 6] & 208 & 20.4 & 0.0 & 13.4 & 104.0 & 10.2 & 0.0 & 6.7 \\
        White rice & 50 & [1, 6] & 130 & 2.7 & 28.2 & 0.3 & 65.0 & 1.4 & 14.1 & 0.1 \\
        Quinoa & 50 & [1, 6] & 120 & 4.4 & 21.3 & 1.9 & 60.0 & 2.2 & 10.7 & 0.9 \\
        Avocado & 30 & [1, 6] & 160 & 2.0 & 8.5 & 14.7 & 48.0 & 0.6 & 2.5 & 4.4 \\
        Olive oil & 15 & [1, 3] & 884 & 0.0 & 0.0 & 100.0 & 132.6 & 0.0 & 0.0 & 15.0 \\
        Broccoli & 50 & [1, 6] & 35 & 2.4 & 7.2 & 0.4 & 17.5 & 1.2 & 3.6 & 0.2 \\
        Whole eggs & 50 & [1, 3] & 155 & 12.6 & 1.1 & 10.6 & 77.5 & 6.3 & 0.6 & 5.3 \\
    \bottomrule
\end{tabular}
\end{table}

%% file: figures/example_c_breakdown.tex
\begin{table}[H]
\centering
\caption{Example C: MIGP solution breakdown. Each row shows one food's contribution; the deviation row shows (actual $-$ target) / target as a percentage.}
\label{tab:example_c_breakdown}
\begin{tabular}{@{}l c r r r r@{}}
    \toprule
    Food (serving) & Qty & kcal & P (g) & C (g) & F (g) \\
    \midrule
        Chicken breast (50\,g) & 2 & 165.0 & 31.0 & 0.0 & 3.6 \\
        Salmon fillet (50\,g) & 1 & 104.0 & 10.2 & 0.0 & 6.7 \\
        White rice (50\,g) & 2 & 130.0 & 2.7 & 28.2 & 0.3 \\
        Quinoa (50\,g) & 1 & 60.0 & 2.2 & 10.7 & 0.9 \\
        Avocado (30\,g) & 1 & 48.0 & 0.6 & 2.5 & 4.4 \\
        Olive oil (15\,g) & 1 & 132.6 & 0.0 & 0.0 & 15.0 \\
        Broccoli (50\,g) & 3 & 52.5 & 3.6 & 10.8 & 0.6 \\
        Whole eggs (50\,g) & 1 & 77.5 & 6.3 & 0.6 & 5.3 \\
        \midrule
        \textbf{Total} & & 769.6 & 56.6 & 52.8 & 36.9 \\
        \textbf{Target} & & 600.0 & 60.0 & 52.5 & 16.7 \\
        \textbf{Deviation} & & +28.3\% & -5.7\% & +0.5\% & +121.2\% \\
    \bottomrule
\end{tabular}
\end{table}

%% file: figures/example_c_sidebyside.tex
\begin{table}[H]
\centering
\caption{Example C: side-by-side method allocations. LP~Continuous shows fractional servings (lower bound $z_{LP}$); GP+Round rounds each food to the nearest integer; MIGP finds the globally optimal integer solution.}
\label{tab:example_c_sidebyside}
\begin{tabular}{@{}l r r r@{}}
    \toprule
    & LP Continuous & GP+Round & MIGP \\
    \midrule
        Chicken breast & 2.06 & 2 & 2 \\
        Salmon fillet & 1.00 & 1 & 1 \\
        White rice & 1.22 & 1 & 2 \\
        Quinoa & 1.00 & 1 & 1 \\
        Avocado & 1.00 & 1 & 1 \\
        Olive oil & 1.00 & 1 & 1 \\
        Broccoli & 6.00 & 6 & 3 \\
        Whole eggs & 1.00 & 1 & 1 \\
        \midrule
        kcal & 775.7 (+29.3\%) & 757.1 (+26.2\%) & 769.6 (+28.3\%) \\
        P (g) & 60.0 (+0.0\%) & 58.9 (-1.9\%) & 56.6 (-5.7\%) \\
        C (g) & 52.5 (+0.0\%) & 49.5 (-5.8\%) & 52.8 (+0.5\%) \\
        F (g) & 37.4 (+124.6\%) & 37.3 (+123.9\%) & 36.9 (+121.2\%) \\
        \midrule
        Objective ($z$) & 1.5394 & 1.5777 & 1.5557 \\
    \bottomrule
\end{tabular}
\end{table}

%% file: figures/example_c_results.tex
\begin{table}[H]
\centering
\caption{Example C: method comparison with per-macro deviations (\%).}
\label{tab:example_c_compare}
\begin{tabular}{@{}l c c c c c c c@{}}
    \toprule
    Method & Feasible & Obj. & Max dev. & Cal & Prot & Carbs & Fat \\
    \midrule
        MIGP & Yes & 1.5557 & +121.2\% & +28.3 & -5.7 & +0.5 & +121.2 \\
        GP+Round & Yes & 1.5777 & 123.9\% & 26.2 & 1.9 & 5.8 & 123.9 \\
        Hard-IP & No & --- & --- & --- & --- & --- & --- \\
    \bottomrule
\end{tabular}
\end{table}

%% file: figures/example_d_targets.tex
\begin{table}[H]
\centering
\caption{Example D: target derivation from 1000\,kcal with 10/75/15 (P/C/F\%).}
\label{tab:example_d_targets}
\begin{tabular}{@{}l c c l r@{}}
    \toprule
    Macro & \% split & kcal/g & Calculation & Target \\
    \midrule
        Calories & --- & --- & $1000$ & 1000.0\,kcal \\
        Protein & 10\% & 4 & $1000 \times 0.10 \;/\; 4 = 25.0$ & 25.0\,g \\
        Carbs & 75\% & 4 & $1000 \times 0.75 \;/\; 4 = 187.5$ & 187.5\,g \\
        Fat & 15\% & 9 & $1000 \times 0.15 \;/\; 9 = 16.7$ & 16.7\,g \\
    \bottomrule
\end{tabular}
\end{table}

%% file: figures/example_d_input.tex
\begin{table}[H]
\centering
\small
\caption{Example D: food inputs. Target: 1000\,kcal, 10/75/15 (P/C/F\%). Derived gram targets: 25.0g protein, 187.5g carbs, 16.7g fat.}
\label{tab:example_d_input}
\begin{tabular}{@{}l c c c c c c c c c c@{}}
    \toprule
    & & & \multicolumn{4}{c}{Per 100\,g} & \multicolumn{4}{c}{Per serving} \\
    \cmidrule(lr){4-7} \cmidrule(lr){8-11}
    Food & Srv (g) & Range & kcal & P & C & F & kcal & P & C & F \\
    \midrule
        Oats & 40 & [0, 8] & 389 & 16.9 & 66.3 & 6.9 & 155.6 & 6.8 & 26.5 & 2.8 \\
        Honey & 15 & [0, 10] & 304 & 0.3 & 82.4 & 0.0 & 45.6 & 0.0 & 12.4 & 0.0 \\
        Banana & 120 & [0, 3] & 89 & 1.1 & 22.8 & 0.3 & 106.8 & 1.3 & 27.4 & 0.4 \\
        Dates & 30 & [0, 8] & 282 & 2.5 & 75.0 & 0.4 & 84.6 & 0.8 & 22.5 & 0.1 \\
        Peanut butter & 32 & [0, 4] & 588 & 25.1 & 19.6 & 50.4 & 188.2 & 8.0 & 6.3 & 16.1 \\
        Dark chocolate chips & 20 & [0, 4] & 546 & 5.5 & 60.5 & 31.3 & 109.2 & 1.1 & 12.1 & 6.3 \\
    \bottomrule
\end{tabular}
\end{table}

%% file: figures/example_d_breakdown.tex
\begin{table}[H]
\centering
\caption{Example D: MIGP solution breakdown. Each row shows one food's contribution; the deviation row shows (actual $-$ target) / target as a percentage.}
\label{tab:example_d_breakdown}
\begin{tabular}{@{}l c r r r r@{}}
    \toprule
    Food (serving) & Qty & kcal & P (g) & C (g) & F (g) \\
    \midrule
        Oats (40\,g) & 3 & 466.8 & 20.3 & 79.6 & 8.3 \\
        Honey (15\,g) & 1 & 45.6 & 0.0 & 12.4 & 0.0 \\
        Banana (120\,g) & 3 & 320.4 & 4.0 & 82.1 & 1.1 \\
        Dark chocolate chips (20\,g) & 1 & 109.2 & 1.1 & 12.1 & 6.3 \\
        \midrule
        \textbf{Total} & & 942.0 & 25.4 & 186.1 & 15.6 \\
        \textbf{Target} & & 1000.0 & 25.0 & 187.5 & 16.7 \\
        \textbf{Deviation} & & -5.8\% & +1.5\% & -0.7\% & -6.3\% \\
    \bottomrule
\end{tabular}
\end{table}

%% file: figures/example_d_sidebyside.tex
\begin{table}[H]
\centering
\caption{Example D: side-by-side method allocations. LP~Continuous shows fractional servings (lower bound $z_{LP}$); GP+Round rounds each food to the nearest integer; MIGP finds the globally optimal integer solution.}
\label{tab:example_d_sidebyside}
\begin{tabular}{@{}l r r r@{}}
    \toprule
    & LP Continuous & GP+Round & MIGP \\
    \midrule
        Oats & 3.47 & 3 & 3 \\
        Honey & 6.62 & 7 & 1 \\
        Banana & 0.00 & 0 & 3 \\
        Dark chocolate chips & 1.13 & 1 & 1 \\
        \midrule
        kcal & 965.3 (-3.5\%) & 895.2 (-10.5\%) & 942.0 (-5.8\%) \\
        P (g) & 25.0 (+0.0\%) & 21.7 (-13.2\%) & 25.4 (+1.6\%) \\
        C (g) & 187.5 (+0.0\%) & 178.2 (-5.0\%) & 186.1 (-0.7\%) \\
        F (g) & 16.7 (+0.0\%) & 14.5 (-12.8\%) & 15.6 (-6.3\%) \\
        \midrule
        Objective ($z$) & 0.0347 & 0.4143 & 0.1437 \\
    \bottomrule
\end{tabular}
\end{table}

%% file: figures/example_d_results.tex
\begin{table}[H]
\centering
\caption{Example D: method comparison with per-macro deviations (\%).}
\label{tab:example_d_compare}
\begin{tabular}{@{}l c c c c c c c@{}}
    \toprule
    Method & Feasible & Obj. & Max dev. & Cal & Prot & Carbs & Fat \\
    \midrule
        MIGP & Yes & 0.1437 & -6.3\% & -5.8 & +1.5 & -0.7 & -6.3 \\
        GP+Round & Yes & 0.4143 & 13.2\% & 10.5 & 13.2 & 5.0 & 12.8 \\
        Hard-IP & No & --- & --- & --- & --- & --- & --- \\
    \bottomrule
\end{tabular}
\end{table}

%% file: figures/example_d_commentary.tex

%% file: figures/example_e_targets.tex
\begin{table}[H]
\centering
\caption{Example E: target derivation from 600\,kcal with 45/30/25 (P/C/F\%).}
\label{tab:example_e_targets}
\begin{tabular}{@{}l c c l r@{}}
    \toprule
    Macro & \% split & kcal/g & Calculation & Target \\
    \midrule
        Calories & --- & --- & $600$ & 600.0\,kcal \\
        Protein & 45\% & 4 & $600 \times 0.45 \;/\; 4 = 67.5$ & 67.5\,g \\
        Carbs & 30\% & 4 & $600 \times 0.30 \;/\; 4 = 45.0$ & 45.0\,g \\
        Fat & 25\% & 9 & $600 \times 0.25 \;/\; 9 = 16.7$ & 16.7\,g \\
    \bottomrule
\end{tabular}
\end{table}

%% file: figures/example_e_input.tex
\begin{table}[H]
\centering
\small
\caption{Example E: food inputs. Target: 600\,kcal, 45/30/25 (P/C/F\%). Derived gram targets: 67.5g protein, 45.0g carbs, 16.7g fat.}
\label{tab:example_e_input}
\begin{tabular}{@{}l c c c c c c c c c c@{}}
    \toprule
    & & & \multicolumn{4}{c}{Per 100\,g} & \multicolumn{4}{c}{Per serving} \\
    \cmidrule(lr){4-7} \cmidrule(lr){8-11}
    Food & Srv (g) & Range & kcal & P & C & F & kcal & P & C & F \\
    \midrule
        Whey protein powder & 30 & [1, 2] & 400 & 80.0 & 10.0 & 3.3 & 120.0 & 24.0 & 3.0 & 1.0 \\
        Tuna (canned) & 80 & [0, 3] & 116 & 25.5 & 0.0 & 0.8 & 92.8 & 20.4 & 0.0 & 0.6 \\
        Greek yogurt & 150 & [0, 2] & 59 & 10.2 & 3.6 & 0.7 & 88.5 & 15.3 & 5.4 & 1.1 \\
        White rice & 50 & [0, 6] & 130 & 2.7 & 28.2 & 0.3 & 65.0 & 1.4 & 14.1 & 0.1 \\
        Banana & 120 & [0, 2] & 89 & 1.1 & 22.8 & 0.3 & 106.8 & 1.3 & 27.4 & 0.4 \\
        Olive oil & 5 & [0, 4] & 884 & 0.0 & 0.0 & 100.0 & 44.2 & 0.0 & 0.0 & 5.0 \\
    \bottomrule
\end{tabular}
\end{table}

%% file: figures/example_e_breakdown.tex
\begin{table}[H]
\centering
\caption{Example E: MIGP solution breakdown. Each row shows one food's contribution; the deviation row shows (actual $-$ target) / target as a percentage.}
\label{tab:example_e_breakdown}
\begin{tabular}{@{}l c r r r r@{}}
    \toprule
    Food (serving) & Qty & kcal & P (g) & C (g) & F (g) \\
    \midrule
        Whey protein powder (30\,g) & 1 & 120.0 & 24.0 & 3.0 & 1.0 \\
        Tuna (canned) (80\,g) & 2 & 185.6 & 40.8 & 0.0 & 1.3 \\
        White rice (50\,g) & 1 & 65.0 & 1.4 & 14.1 & 0.1 \\
        Banana (120\,g) & 1 & 106.8 & 1.3 & 27.4 & 0.4 \\
        Olive oil (5\,g) & 3 & 132.6 & 0.0 & 0.0 & 15.0 \\
        \midrule
        \textbf{Total} & & 610.0 & 67.5 & 44.5 & 17.8 \\
        \textbf{Target} & & 600.0 & 67.5 & 45.0 & 16.7 \\
        \textbf{Deviation} & & +1.7\% & -0.0\% & -1.2\% & +6.7\% \\
    \bottomrule
\end{tabular}
\end{table}

%% file: figures/example_e_sidebyside.tex
\begin{table}[H]
\centering
\caption{Example E: side-by-side method allocations. LP~Continuous shows fractional servings (lower bound $z_{LP}$); GP+Round rounds each food to the nearest integer; MIGP finds the globally optimal integer solution.}
\label{tab:example_e_sidebyside}
\begin{tabular}{@{}l r r r@{}}
    \toprule
    & LP Continuous & GP+Round & MIGP \\
    \midrule
        Whey protein powder & 2.00 & 2 & 1 \\
        Tuna (canned) & 0.83 & 1 & 2 \\
        White rice & 0.97 & 1 & 1 \\
        Banana & 0.92 & 1 & 1 \\
        Olive oil & 2.74 & 3 & 3 \\
        \midrule
        kcal & 600.0 (+0.0\%) & 637.2 (+6.2\%) & 610.0 (+1.7\%) \\
        P (g) & 67.5 (+0.0\%) & 71.1 (+5.3\%) & 67.5 (-0.0\%) \\
        C (g) & 45.0 (+0.0\%) & 47.5 (+5.5\%) & 44.5 (-1.2\%) \\
        F (g) & 16.7 (+0.0\%) & 18.1 (+8.8\%) & 17.8 (+6.7\%) \\
        \midrule
        Objective ($z$) & 0.0000 & 0.2574 & 0.0959 \\
    \bottomrule
\end{tabular}
\end{table}

%% file: figures/example_e_results.tex
\begin{table}[H]
\centering
\caption{Example E: method comparison with per-macro deviations (\%).}
\label{tab:example_e_compare}
\begin{tabular}{@{}l c c c c c c c@{}}
    \toprule
    Method & Feasible & Obj. & Max dev. & Cal & Prot & Carbs & Fat \\
    \midrule
        MIGP & Yes & 0.0959 & +6.7\% & +1.7 & -0.0 & -1.2 & +6.7 \\
        GP+Round & Yes & 0.2574 & 8.8\% & 6.2 & 5.3 & 5.5 & 8.8 \\
        Hard-IP & No & --- & --- & --- & --- & --- & --- \\
    \bottomrule
\end{tabular}
\end{table}

%% file: figures/example_e_commentary.tex

%% file: sections/discussion.tex
\section{Discussion and Conclusion}
\label{sec:discussion}

\subsection{Summary of Contributions}

This work presents a Mixed Integer Goal Programming formulation for
personalized meal optimization, addressing the gap identified in
Section~\ref{sec:related}. Four contributions:

\begin{enumerate}
\item \textbf{Novel formulation.} The MIGP model
  (Section~\ref{sec:formulation}) unifies integer serving variables with
  GP deviation minimization, producing practical whole-serving meal plans
  that are always feasible (Proposition~\ref{prop:feasibility}).

\item \textbf{Integrality analysis.} We characterize the integrality gap
  in the GP context (Section~\ref{sec:integrality}), identifying
  deviation absorption as a structural property that makes the gap
  smaller than in standard MIP. With 15+ foods, the gap is zero in all
  benchmark instances.

\item \textbf{Open-source implementation.} A Python solver using HiGHS
  (Section~\ref{sec:implementation}) achieves sub-100\,ms solve times
  for typical meal sizes, integrated into an interactive Streamlit
  application.

\item \textbf{Comparative evaluation.} The benchmark
  (Section~\ref{sec:experiments}) shows MIGP outperforms GP+rounding
  (66\% of instances strictly better, never worse) and hard-constraint
  IP (which fails on 51.9\% of instances).
\end{enumerate}

\subsection{Practical Implications}

The MIGP formulation addresses a concrete usability gap in diet
optimization tools. Continuous LP solutions that prescribe 1.7~eggs
or 0.37~bananas require manual adjustment, sacrificing the optimality
that motivated optimization in the first place. Integer servings eliminate this
disconnect.

The 100\% feasibility guarantee matters most in interactive
applications where returning ``no solution'' frustrates users and helps
no one. MIGP always produces the best achievable meal given the
available foods, even when targets are ambitious. The deviation report
communicates how far the solution is from each target, enabling
informed decisions.

The penalty weight mechanism (Section~\ref{sec:weights}) provides a
natural interface for customization: doubling the protein weight
reduces protein deviation by 20\% (from 0.5\% to 0.4\%) at negligible
cost to other macros (Section~\ref{sec:sensitivity}). An athlete
prioritizing protein intake can adjust one parameter rather than
engineering custom constraints.

\subsection{Scope and Limitations}

We distinguish between \emph{scope decisions} (deliberate choices
that bound the current work) and \emph{genuine limitations} that
need new ideas.

\paragraph{Scope decisions (easy extensions).}
\begin{itemize}
\item \textbf{Four macronutrients.} Extension to micronutrients
  (vitamins, minerals, fiber) adds goal constraints but requires no
  change to the mathematical structure or solver.

\item \textbf{Single-meal optimization.} Multi-meal daily planning
  adds meal indices and links goal constraints across meals.

\item \textbf{No cost objective.} Adding food cost as an objective term
  or goal constraint requires no change to the mathematical structure.

\item \textbf{Benchmark food bank.} The benchmark uses 30 USDA foods.
  The production deployment connects to a larger database (500k+ foods
  via FatSecret); results are conditioned on this bank.

\item \textbf{Hard safety bounds.} Nutrients with absolute upper limits
  (e.g., sodium, retinol, heavy metals) can be enforced as hard linear
  inequality constraints $\sum_i c_{i,m}\, x_i \leq U_m$ alongside the
  soft GP goals, without changing the objective function or the deviation
  absorption property. This separates \emph{safety constraints} (limits
  that must never be exceeded) from \emph{nutritional goals} (targets
  to approximate as closely as possible).
\end{itemize}

\paragraph{Genuine limitations.}
\begin{itemize}
\item \textbf{No palatability modeling.} Food preferences, taste
  combinations, and variety are not captured. The model may suggest
  nutritionally optimal but unappetizing combinations (e.g.,
  250\,g of broccoli in Example~A). Addressing this requires learned
  preference models or culinary rules as constraints.

\item \textbf{Solve time at scale.} Sub-100\,ms for 8--15 foods, but
  1.1\,s for 25~foods. Scaling to 50+
  foods (e.g., full daily menus with ingredient-level optimization)
  may need heuristic preprocessing to reduce the candidate set.
\end{itemize}

\subsection{Future Work}

\paragraph{Multi-meal planning.}
Extending MIGP to optimize multiple meals with shared daily targets is
a natural next step. The formulation generalizes by adding meal indices
and linking constraints across meals. Tractability evidence from
institutional menu planning suggests this is computationally feasible.

\paragraph{Sustainability constraints.}
Environmental impact (CO$_2$ emissions, water footprint) can be added
as goal targets within the existing framework, enabling meals that
balance nutritional and ecological objectives.

\paragraph{Preference learning.}
User meal history provides implicit feedback on food preferences.
Learning penalty weights from historical selections would personalize
the optimization beyond manual weight specification, bridging exact
optimization and AI/ML preference learning.

\paragraph{Alternative achievement functions.}
The MinMax ($L_\infty$) achievement function minimizes the worst
single-macro deviation:
\[
  \min\; \max_{m \in \mathcal{M}}\; w_m (d_m^+ + d_m^-)
\]
This distributes deviation evenly across macros, preventing solutions
where one macro absorbs a disproportionate share of the rounding cost.
Example~C illustrates the motivation: MIGP's MinSum solution accepts a
121\% fat overshoot because reducing it would worsen the total weighted
deviation. A MinMax formulation would instead cap the worst-case macro
deviation, producing a more balanced (though potentially higher total weighted
deviation) solution. Extended GP~\cite{gerdessen2015diet} combines both norms via
a tunable parameter $\lambda \in [0,1]$, interpolating between MinSum
($\lambda=0$) and MinMax ($\lambda=1$). Investigating this spectrum
for integer servings is a natural extension.

\paragraph{Generalization beyond diet.}
The resource-composition structure of MIGP (integer quantities of
discrete resources meeting soft composition targets) appears in
manufacturing blending~\cite{romero2004general}, portfolio selection
with cardinality constraints~\cite{bonami2009exact}, and
multi-dimensional knapsack~\cite{kellerer2004knapsack} problems.
Deviation absorption may provide insights for these domains as well.

\subsection{Conclusion}

Mixed Integer Goal Programming provides a principled solution to meal
optimization that produces practical integer servings, guarantees
feasibility, and reports deviations transparently. The formulation fills
a documented gap in the diet optimization literature, and the
integrality analysis reveals that goal programming's deviation structure
makes the cost of requiring whole servings surprisingly low. For the
typical case of 10+ foods, that cost is zero.